\newif\ifarxiv
\arxivtrue

\documentclass[11pt]{article}
\usepackage{hyperref}
\hypersetup{
    colorlinks,
    linkcolor={blue!60!black},
    citecolor={blue!60!black},
    urlcolor={blue!60!black}
}
\usepackage[numbers]{natbib}
\usepackage{fullpage}

\usepackage{pgfplotstable}
\usepackage{graphicx}
\usepackage{float}
\usepackage{booktabs}

\providecommand{\comment}[1]{}

\usepackage{enumerate}
\usepackage{amssymb}
\usepackage{amsbsy}
\usepackage{amsmath}
\usepackage{amsthm}
\usepackage{amsfonts}
\usepackage{latexsym}
\usepackage{graphicx}
\usepackage{xifthen}
\usepackage{xspace}
\usepackage{bbm}
\newcommand{\cf}{{\it cf.}}
\newcommand{\eg}{{\it e.g.}}

\newcommand{\mc}[1]{\mathcal{#1}}

\newcommand{\norm}[1]{\left\|{#1}\right\|} %
\newcommand{\ltwo}[1]{\norm{#1}_2} %
\newcommand{\linf}[1]{\norm{#1}_\infty} %

\newcommand{\defeq}{:=}

\newcommand{\half}{\frac{1}{2}}
\newcommand{\R}{\mathbb{R}}
\makeatletter
\long\def\@makecaption#1#2{
  \vskip 0.8ex
  \setbox\@tempboxa\hbox{\small {\bf #1:} #2}
  \parindent 1.5em  %
  \dimen0=\hsize
  \advance\dimen0 by -3em
  \ifdim \wd\@tempboxa >\dimen0
  \hbox to \hsize{
    \parindent 0em
    \hfil
    \parbox{\dimen0}{\def\baselinestretch{0.96}\small
      {\bf #1.} #2
    }
    \hfil}
  \else \hbox to \hsize{\hfil \box\@tempboxa \hfil}
  \fi
}
\makeatother

\newcommand{\openright}[2]{\left[{#1},{#2}\right)} %

\newcommand{\E}{\mathbb{E}} %
\renewcommand{\P}{\mathbb{P}} %
\newcommand{\var}{{\rm Var}} %
\newcommand{\simiid}{\stackrel{\rm i.i.d.}{\sim}}

\newcommand{\fdiv}[2]{D_f\left({#1} |\!| {#2} \right)}  %
\providecommand{\argmin}{\mathop{\rm argmin}}
\providecommand{\dom}{\mathop{\rm dom}}

\newcommand{\hinge}[1]{\left[{#1}\right]_+} %

\providecommand{\minimize}{\mathop{\rm minimize}}
\providecommand{\maximize}{\mathop{\rm maximize}}

\newtheorem{theorem}{Theorem}

\newtheorem{proposition}[theorem]{Proposition}

\newtheorem{lemma}{Lemma}

\newtheorem*{claim*}{Claim}

\renewenvironment{proof}{\noindent{\bf Proof}\hspace*{1em}}{\qed\bigskip\\}

\newenvironment{proof-sketch}{\noindent{\bf Sketch of Proof}
  \hspace*{1em}}{\qed\bigskip\\}
\newenvironment{proof-idea}{\noindent{\bf Proof Idea}
  \hspace*{1em}}{\qed\bigskip\\}

\newenvironment{proof-of-claim}{\noindent{\bf Proof of Claim}
  \hspace*{1em}}{\qed\bigskip\\}

\newenvironment{proof-of-lemma}[1][{}]{\noindent{\bf Proof of Lemma {#1}}
  \hspace*{1em}}{\qed\bigskip\\}
\newenvironment{proof-of-proposition}[1][{}]{\noindent{\bf
    Proof of Proposition {#1}}
  \hspace*{1em}}{\qed\bigskip\\}
\newenvironment{proof-of-theorem}[1][{}]{\noindent{\bf Proof of Theorem {#1}}
  \hspace*{1em}}{\qed\bigskip\\}
\newenvironment{inner-proof}{\noindent{\bf Proof}\hspace{1em}}{
  $\bigtriangledown$\medskip\\}
\newenvironment{proof-attempt}{\noindent{\bf Proof Attempt}
  \hspace*{1em}}{\qed\bigskip\\}

\newcounter{example}

\newenvironment{example*}[1][]{
  \ifthenelse{\isempty{#1}}{%
    \noindent \textbf{Example:}\hspace*{.05em}
  }{%
    \noindent \textbf{Example} ({#1})\textbf{:}\hspace*{.05em}
  }
}{%
  $\clubsuit$ \bigskip
}

\newcounter{remark}

\newenvironment{remark*}[1][]{
  \ifthenelse{\isempty{#1}}{%
    \noindent \textbf{Remark:}\hspace*{.05em}
  }{%
    \noindent \textbf{Remark} ({#1})\textbf{:}\hspace*{.05em}
  }
}{%
  $\diamondsuit$ \bigskip
}

\newcommand{\tol}{\rho}  %

\usepackage{subfigure}
\usepackage{color}
\usepackage{algorithm}
\usepackage{algpseudocode}
\usepackage{tabularx}

\usepackage{amsfonts}

\usepackage{manyfoot}

\DeclareNewFootnote{A}
\DeclareNewFootnote{B}

\let\footnoteR\footnoteB
\let\footnote\footnoteA

\usepackage{caption}
\usepackage{graphicx}

\definecolor{matlab_blue}{rgb}{0,    0.4470,   0.7410}
\definecolor{matlab_red}{rgb}{0.8500,    0.3250,    0.0980}
\definecolor{matlab_yellow}{rgb}{0.9290,    0.6940,    0.1250}
\definecolor{matlab_purple}{rgb}{0.4940,    0.1840,    0.5560}
\definecolor{matlab_green}{rgb}{0.4660,    0.6740,    0.1880}

\renewcommand{\fdiv}[2]{D_{\phi}\left({#1} |\!| {#2} \right)}  %

\title{FormulaZero: Distributionally Robust Online Adaptation via \\
		Offline Population Synthesis}

\makeatletter
\long\def\@makecaption#1#2{
  \vskip 0.8ex
  \setbox\@tempboxa\hbox{\small {\bf #1:} #2}
  \parindent 1.5em  %
  \dimen0=\hsize
  \advance\dimen0 by -3em
  \ifdim \wd\@tempboxa >\dimen0
  \hbox to \hsize{
    \parindent 0em
    \hfil 
    \parbox{\dimen0}{\def\baselinestretch{0.96}\small
      {\bf #1.} #2
    } 
    \hfil}
  \else \hbox to \hsize{\hfil \box\@tempboxa \hfil}
  \fi
}
\makeatother

\begin{document}
\abovedisplayskip=8pt plus0pt minus3pt
\belowdisplayskip=8pt plus0pt minus3pt

\begin{center}
  {\LARGE FormulaZero: Distributionally Robust Online Adaptation via\\Offline Population Synthesis} \\
  \vspace{.5cm} {\Large Aman Sinha\footnoteR{Equal contribution}$^{1}$~~Matthew O'Kelly$^{*2}$~~Hongrui Zheng$^{*2}$\\ Rahul Mangharam$^{2}$~~John Duchi$^{1}$~~Russ Tedrake$^{3}$}\\
  \vspace{.2cm}
  $^{1}${\large Stanford University}\\
  $^2${\large University of Pennsylvania}\\
  $^3${\large Massachusetts Institute of Technology}\\
  \vspace{.2cm} \texttt{amans@stanford.edu, \{mokelly,hongruiz,rahulm\}@seas.upenn.edu, jduchi@stanford.edu, russt@mit.edu}
\end{center}

\newif\ifarxiv
\arxivtrue
\begin{abstract}

Balancing performance and safety is crucial to deploying autonomous vehicles
in multi-agent environments. In particular, autonomous racing is a domain 
that penalizes safe but conservative policies, highlighting the need for
robust, adaptive strategies. Current approaches either make simplifying
assumptions about other agents or lack robust mechanisms for online
adaptation. This work makes algorithmic contributions to both
challenges. First, to generate a realistic, diverse set of opponents, we
develop a novel method for self-play based on replica-exchange Markov chain
Monte Carlo. Second, we propose a distributionally robust bandit
optimization procedure that adaptively adjusts risk aversion
relative to uncertainty in beliefs about opponents’ behaviors. We rigorously
quantify the tradeoffs in performance and robustness when approximating
these computations in real-time motion-planning, and we
demonstrate our methods experimentally on autonomous vehicles
that achieve scaled speeds comparable to Formula One racecars.

\end{abstract}

\ifarxiv
\else
\vspace{-20pt}
\fi
\section{Introduction}

Current autonomous vehicle (AV) technology still struggles in competitive multi-agent scenarios, such as merging onto a highway, where both maximizing performance (negotiating the merge without delay or hesitation) and maintaining safety (avoiding a crash) are important. 
The strategic implications of this tradeoff are magnified in racing. During the 2019 Formula One season, the race-winner achieved the fastest lap in only 33\% of events \cite{formulaone}. %
Empirically, 
the weak correlation between achieving the fastest lap-time and winning suggests that consistent and robust performance is critical to success. In this paper, we investigate this intuition in the setting of autonomous racing (AR). 
In AR, an AV must lap a racetrack in the presence of other agents deploying unknown policies. The agent wins if it completes the race faster than its opponents; a crash automatically results in a loss.
AR is a competitive multi-agent game, a general setting challenging for a number of reasons, especially in robotics applications. 
First, failures are expensive and dangerous, so learning-based approaches must avoid such behavior or rely on simulation while training. Second, the agents only partially observe their opponent's state, and %
these observations do not uniquely determine the opponent's behavior. Finally, the agents must make decisions online; the opponent's strategy is a tightly-held secret and cannot be obtained by collecting data before the competition.

\ifarxiv
\paragraph{Problem:}
\else
\textbf{Problem:}
\fi
We frame the AR challenge in the context of robust reinforcement learning. We analyze the system as a partially-observed Markov decision process (POMDP) $(\mathcal S, \mathcal A, P_{sa}, \mathcal O, r, \lambda)$, with state space $\mathcal S$, action space $\mathcal A$, state-action transition probabilities $P_{sa}$, observation space $\mathcal O$, rewards $r : \mathcal O \to \R$, and discount factor $\lambda$. Furthermore, we capture uncertainty in behaviors of other agents through an \textit{ambiguity}\footnote{Ambiguity is a synonym for uncertainty~\cite{gilboa2016ambiguity}. Formal descriptions in this paper use the term ambiguity.} set $\mathcal P$ for the state-action transitions. Then the AV's objective is
\begin{equation}
\maximize \inf_{P_{sa}\in \mathcal P} \sum_t \lambda^t \E [r(o(t))].
\label{eq:robrl}
\end{equation}
The obvious price of robustness \cite{bertsimas2004price} 
is that a larger ambiguity set ensures a greater degree of safety while sacrificing performance against a particular opponent. If we knew the opponent's behavior, we would need no ambiguity set; equivalently, the ambiguity set would shrink to the nominal state-action transition distribution. Our goal is to automatically trade between performance and robustness as we play against opponents, which breaks down into two challenges: parametrizing the ambiguity set to allow tractable inference and computing the robust cost efficiently online.

\ifarxiv
\paragraph{Contributions:}
\else
\textbf{Contributions:}
\fi
The paper has three contributions: (i) a novel population-based self-play method to parametrize opponent behaviors, (ii) a provably efficient approach to estimate the ambiguity set and the robust cost online, and (iii) a demonstration of these methods on real autonomous vehicles. The name of our approach---FormulaZero---alludes both to the Formula One racing league and the fact that we use self-play (and no demonstrations) to learn competitive behaviors, similar to the approach of AlphaZero \cite{silver2018general}.

Section~\ref{sec:related} gives context to our learning problem, including connections to classical control techniques. In Section~\ref{sec:pbt}, we describe the first challenge: learning how to parametrize the ambiguity set $\mc{P}$. Rather than directly consider the continuous action space of throttle and steering outputs, we synthesize a library of ``prototype'' opponent behaviors offline using population-based self-play. When racing against a particular opponent, the agent maintains a belief vector $w(t)$ of the opponent's behavior patterns as a categorical distribution over these prototype behaviors. We then parametrize the ambiguity set as a ball around this nominal belief $w(t)$.

The second challenge, presented in Section~\ref{sec:bandits}, is an online optimization problem, wherein the agent iteratively updates the ambiguity set (\eg~updates $w(t)$) and computes the robust cost of this set. In other words, the agent attempts to learn the opponent's behavior online to maximize its competitive performance. Since this optimization occurs on a moving vehicle with limited computational resources, we provide convergence results that highlight tradeoffs of performance and robustness with respect to these budgets.
Finally, Section~\ref{sec:experiments} details the practical implications of the theoretical results, emergent properties of the method, and the experimental performance of our approach. %

\subsection{Related work}
\label{sec:related}
Reinforcement learning (RL) has achieved unprecedented success on
classic two-player games \citep[\eg][]{silver2018general}, leading to new approaches in partially-observable games with continuous action spaces~\cite{arulkumaran2019alphastar, berner2019dota}. In these works, agents train via self-play using Monte Carlo tree search \cite{browne2012survey,sutton2018reinforcement} or population-based methods \cite{jaderberg2017population,jaderberg2019human}. The agents  optimize expected performance rather than adapt to individual variations in opponent strategy, which can lead to poor performance against particular opponents~\cite{bansal2017emergent}. In contrast, our method explicitly incorporates adaptivity to opponents.

Robust approaches to RL and control (like this work) explicitly model uncertainty. In RL, this amounts to planning in a robust MDP~\cite{nilim2005robust} or a POMDP~\cite{kaelbling1998planning}. Early results \citet{bagnell2001solving} and \citet{nilim2005robust} describe solutions for robust planning in (PO)MDPs with tabular state/action spaces. Equivalent results in control are analytical formulations applicable to uncertainty in linear time-invariant systems~\cite{doyle1988state, vinnicombe1993frequency, zhou1996robust}. Recent works~\cite{tamar2014scaling, pinto2017robust, mandlekar2017adversarially, gleave2019adversarial} describe minimax and adversarial RL frameworks for nonlinear systems and continuous action spaces. Like our approach, these methods fall broadly under the framework of robust optimization. Unlike these works, which consider worst-case planning under a fixed uncertainty distribution, our approach updates the distribution online. 

Our approach is designed to adjust the agent's evaluation of short-term plans relative to uncertainty in the opponent's behavior rather than provide worst-case guarantees. Complementary to and compatible with our approach are techniques which provide the latter guarantees, such as robust model predictive control \cite{bemporad1999robust}. Extensions of robust control for nonlinear systems and complex uncertainty models are also compatible (\eg~\citet{majumdar2013robust,althoff2014online,gao2014tube}). In contrast to formal approaches which explicitly guarantee robustness, some authors have proposed multitask or meta-learning approaches (\eg~\citet{caruana1997multitask, he2016opponent, finn2018probabilistic}) can implicitly learn to play against multiple opponents. However, such techniques do not explicitly model uncertainty or quantify robustness, which we deem necessary in the high-risk, safety-critical regime.

Planning in belief space is closely related to our approach and is well-studied in robotics \citep[see \eg][]{kochenderfer2015decision}. Specifically in the AV domain, \citet{galceran2015multipolicy} and \citet{ding2019online} use a Bayesian approach to plan trajectories for AVs in belief space; like this work, both of these approaches characterize the other agent's behavior in the environment categorically. Also similar to this work, \citet{van2011lqg} use a sampled set of goals obtained by planning from other agents' perspectives. 
The main difference in this work from standard belief-space planning formulations is inspired by recent results from distributionally robust optimization (DRO) in supervised-learning settings \cite{ Ben-TalHeWaMeRe13, NamkoongDu17}. These methods reweight training data to reduce the variance of the training loss \cite{NamkoongDu17}. While others apply DRO to episodic RL for training \emph{offline} \cite{SinhaNaDu17,smirnova2019distributionally}, we reweight the belief  \emph{online}. %

Online methods for control fall under the umbrella of adaptive control~\citep{kumar1985survey, aastrom2013adaptive}. \citet{dean2018regret} and \citet{agarwal2019online} establish regret bounds for adaptive control methods applied to LTI systems, tightening the relationship to online learning. Due to the more general nature of our problem, we draw from the adversarial multi-armed bandit framework of online learning~\cite{abernethy2009beating, bubeck2012regret, shalev2012online}.

Our belief state corresponds to a categorical distribution of
polices governing an opponent's next action; the goal is to predict which
strategy the opponent is using and compute the best response. This approach is
similar to game-theoretic methods for AR and AV decision making that use the
standard heuristic of iterated best response.  Our work is distinct from
previous work, which either assumes that all agents act with respect to the
same cost function, simplifying the structure of the
game~\citep{liniger2019noncooperative,wang2019game}; or, without this
simplifying assumption, that uses demonstrations to learn possible sets of
policies~\citep{sadigh2016planning,williams2017autonomous}. In constrast, we
learn the set of policies without demonstrations and use DRO to robustly score the AV's plans.
We convert the problem of predicting opponent behavior in a continuous action space into an adversarial bandit problem by learning a set of cost functions that characterize a discrete set of policies. As a result, we would like the opponent models to be both near-optimal and diverse. 
We use determinantal point processes (DPPs) \cite{kulesza2012determinantal} to sample diverse configurations of the parameter space. However,  first we must learn a DPP kernel, which requires that we efficiently sample \emph{competitive} cost functions from the larger configuration space. Since we assume no structure to the set of competitive cost functions, we employ a Markov-chain Monte Carlo (MCMC) method. Complementary methods include variational-inference (\eg~\citet{arenz2018efficient}) and evolutionary (\eg~\citet{mouret2015illuminating}) approaches, which can be challenging to scale up to unstructured, high-dimensional settings of which we have little prior domain knowledge. In our approach, we update the classic simulated tempering method~\cite{marinari1992simulated} with a novel annealing scheme \cite{kirkpatrick1983optimization, vcerny1985thermodynamical} designed for population diversity. We describe this approach next.

\section{Offline population synthesis}
\label{sec:pbt}

The goal of offline population synthesis is to generate a diverse set of
competitive agent behaviors. Formally,
we would like to sample pairs $(x, \theta) \in \mc{X}
\times \Theta$ that are both diverse as well as achieve small values for a
function $f(x, \theta)$.
In our AV application, $\theta$
parametrizes a neural network used to sample trajectories to follow, $x$ is
a weighting of various cost functions that the vehicle uses to select
trajectories from the samples, and $f$ is the simulated lap time.
With this motivation, we treat the method in more generality
assuming (as in our application) that while we can differentiate
$f(x, \theta)$ with respect to $\theta$, $x$ represents hyperparameters and
admits only function evaluations $f(x, \theta)$ rather than first-order
developments. The key
challenge is that we do not \emph{a priori} know a metric
with which to evaluate diversity (\eg, a kernel for a DPP) nor do we
know a base value of $f$ that is deemed acceptable for competitive
performance.

We make this problem more tractable via temperature-based Markov-chain Monte
Carlo (MCMC) and annealing methods
\citep{matyas1965random,hastings1970monte,kirkpatrick1983optimization,vcerny1985thermodynamical,Ingber93,HuHu11}.
Our goal is to sample from a Boltzmann distribution $g(x, \theta; \beta(t))
\propto e^{-\beta(t) f(x, \theta)}$, where $\beta(t)$ is an inverse
``temperature'' parameter that grows (or ``anneals'') with iterations
$t$. When $\beta(t)=0$, all configurations $(x, \theta)$ are equally likely
and all MCMC proposals are accepted; as $\beta(t)$ increases, accepted
proposals favor smaller $f$. Unlike standard hyperparameter optimization
methods \cite{bergstra2012random,jaderberg2017population} that aim to find a
single near-optimal configuration, our goal is to sample a diverse
population of $(x, \theta)$  achieving small $f(x, \theta)$. As such, our
approach---annealed adaptive population tempering
(\textsc{AAdaPT})---maintains a population of configurations and employs
high-exploration proposals based on the classic hit-and-run algorithm
\cite{smith1984efficient,belisle1993hit,lovasz1999hit}.

\subsection{\textsc{AAdaPT}}

\textsc{AAdaPT} builds upon replica-exchange MCMC, also called parallel tempering, which is a standard approach to maintaining a population of configurations~\cite{swendsen1986replica, geyer1991markov}. In parallel tempering, one maintains replicas of the system at $L$  different temperatures $\beta_{1} \ge \beta_{2} ... \ge \beta_{L}$ (which are predetermined and fixed), defining the density of the total configuration as $\prod_{i=1}^L g(x^i, \theta^i; \beta_i)$. The configurations at each level perform standard MCMC steps (also called ``vertical'' steps) as well as ``horizontal'' steps wherein particles are swapped between adjacent temperature levels (see Figure~\ref{fig:aadapt}). Horizontal proposals consist of swapping two configurations in adjacent temperature levels uniformly at random; the proposal is accepted using standard Metropolis-Hastings (MH) criteria \cite{hastings1970monte}. The primary benefit of maintaining parallel configurations is that the configurations at ``colder'' levels (higher $\beta$) can exploit high-exploration moves from ``hotter'' levels (lower $\beta$) which ``tunnel'' down during horizontal steps \cite{geyer1991markov}. This approach allows for faster mixing times, particularly when parallel MCMC proposals occur concurrently in a distributed computing environment.

\ifarxiv
\paragraph{Maintaining a population:}
\else
\textbf{Maintaining a population:}
\fi
In \textsc{AAdaPT} (Algorithm~\ref{alg:aadapt}), we maintain a \emph{population} of $D$ configurations at each separate temperature level. Note that this design always maintains $D$ individuals at the highest performance level (highest $\beta$). The overall configuration density is $\prod_{i=1}^L\prod_{j=1}^D g(x^{i,j}, \theta^{i,j}; \beta_i(t))$. Similar to parallel tempering, horizontal proposals are chosen uniformly at random from configurations at adjacent temperatures (see Appendix~\ref{sec:ops}). We get the same computational benefits of fast mixing in distributed computing environments and a greater ability to exploit high-temperature ``tunneling'' due to the greater number of possible horizontal exchanges between adjacent temperature levels. The benefit of the horizontal steps is even more pronounced in the RL setting as only 
vertical steps require new evaluations of $f$ (\eg~simulations).

\ifarxiv
\begin{algorithm}[t]
	\caption{\textsc{AAdaPT}}
	\label{alg:aadapt}
	\begin{small}
	\begin{algorithmic}
		\State {\bfseries input:} annealing parameter $\alpha$, vertical steps $V$, horizontal exchange steps $E$, temperature levels $L$, population size $d$, initial samples $\{x^{i,j}, \theta^{i,j}\}_{i\in\{1,L\}}^{j \in \{1,D\}}$, iterations $T$
		\State{Evaluate $f(x^{i,j},\theta^{i,j})$}
		\State{{\bfseries for }$t=1$ {\bfseries to} $T$}
		\State{$\;\;\;\;${\bfseries for } $j=1$ {\bfseries to} $L$ {\bfseries do} anneal $\beta_{L-j+1}(t)$ (problem \eqref{eq:cvx_anneal})}
		\State{$\;\;\;\;${\bfseries for } $k=1$ {\bfseries to} $V$ {asynchronously, in parallel}}
		\State{$\;\;\;\;\;\;\;\;${\bfseries for } each population $i$ asynchronously, in parallel}
		\State{$\;\;\;\;\;\;\;\;\;\;\;\;$Sample $\hat x^{i,j}$ according to hit-and-run proposal}
		\State{$\;\;\;\;\;\;\;\;\;\;\;\;$Evaluate $f(\hat x^{i,j},\theta^{i,j})$}
		\State{$\;\;\;\;\;\;\;\;\;\;\;\;$Apply MH criteria to update $x^{i,j}$}
		\State{$\;\;\;\;\;\;\;\;\;\;\;\;$Train $\theta^{i,j}$ via SGD}
		\State{$\;\;\;\;${\bfseries for }$e=1$ {\bfseries to} $E$ {\bfseries do} horizontal swaps (Appendix~\ref{sec:ops})}
	\end{algorithmic}
	\end{small}
\end{algorithm}
\else
\begin{algorithm}[t]
	\caption{\textsc{AAdaPT}}
	\label{alg:aadapt}
	\begin{small}
	\begin{algorithmic}
		\STATE {\bfseries input:} annealing parameter $\alpha$, vertical steps $V$, horizontal exchange steps $E$, temperature levels $L$, population size $d$, initial samples $\{x^{i,j}, \theta^{i,j}\}_{i\in\{1,L\}}^{j \in \{1,D\}}$, iterations $T$
		\STATE{Evaluate $f(x^{i,j},\theta^{i,j})$}
		\STATE{{\bfseries for }$t=1$ {\bfseries to} $T$}
		\STATE{$\;\;\;\;${\bfseries for } $j=1$ {\bfseries to} $L$ {\bfseries do} anneal $\beta_{L-j+1}(t)$ (problem \eqref{eq:cvx_anneal})}
		\STATE{$\;\;\;\;${\bfseries for } $k=1$ {\bfseries to} $V$ {asynchronously, in parallel}}
		\STATE{$\;\;\;\;\;\;\;\;${\bfseries for } each population $i$ asynchronously, in parallel}
		\STATE{$\;\;\;\;\;\;\;\;\;\;\;\;$Sample $\hat x^{i,j}$ according to hit-and-run proposal}
		\STATE{$\;\;\;\;\;\;\;\;\;\;\;\;$Evaluate $f(\hat x^{i,j},\theta^{i,j})$}
		\STATE{$\;\;\;\;\;\;\;\;\;\;\;\;$Apply MH criteria to update $x^{i,j}$}
		\STATE{$\;\;\;\;\;\;\;\;\;\;\;\;$Train $\theta^{i,j}$ via SGD}
		\STATE{$\;\;\;\;${\bfseries for }$e=1$ {\bfseries to} $E$ {\bfseries do} horizontal swaps (Appendix~\ref{sec:ops})}
	\end{algorithmic}
	\end{small}
\end{algorithm}
\fi

\ifarxiv
\paragraph{High-exploration vertical proposals:}
\else
\textbf{High-exploration vertical proposals:}
\fi
Another benefit of maintaining
parallel populations is to improve exploration. We further improve
exploration by using hit-and-run proposals \cite{smith1984efficient,
  belisle1993hit, lovasz1999hit} for the vertical MCMC chains. Namely, from
a current point $(x, \theta)$ we sample a uniformly random direction $\hat{u}$ and then choose a point uniformly on the segment $\mc{X} \cap (\{x +
\mathbb{R} \cdot \hat{u}\} \times \{\theta\})$. This approach has several
guarantees for efficient mixing~\citep{lovasz1999hit,
  lovasz2003hit, lovasz2006hit}. Note that in our implementation the MCMC
steps are only performed on $x$, while $\theta$ updates occur via SGD (see
below).

\ifarxiv
\paragraph{Adaptively annealed temperatures:}
\else
\textbf{Adaptively annealed temperatures:}
\fi
A downside to parallel tempering
is the need to determine the temperature levels $\beta_i$ beforehand. In
\textsc{AAdaPT}.  we adaptively update temperatures. Specifically, we anneal
the prescribed horizontal acceptance probability of particle exchanges
between temperature levels as $\alpha^{t/(L-1)}$ for a fixed hyperparameter
$\alpha \in (0,1)$. Define the empirical acceptance probability of swaps of
configurations between levels $i-1$ and $i$ as
\begin{align*}
  p_{i-1,i} &:= \frac{1}{D^2} \sum_{j=1}^{D}\sum_{k=1}^D (y_{i-1,i}^{j,k})^{\beta_{i-1} - \beta_{i}}\\
  y_{i-1,i}^{j,k} &:= \min \left (1, e^{ f(x^{i-1,j}, \theta^{i-1,j}) - f(x^{i,k}, \theta^{i,k})  } \right ).
\end{align*}
Then, at the beginning of each iteration (in which we perform a series of vertical and horizontal MCMC steps), we update the $\beta_i(t)$ sequentially; we fix $\beta_L(t):= \beta_L = 0$ and for a given $\beta_{i}$, we set $\beta_{i-1}$ by solving the following convex optimization problem:
\begin{equation}
  \minimize_{ \{ \beta_{i-1}\ge \beta_i, ~~p_{i-1,i} \le \alpha^{\frac{t}{(L-1)}} \} } ~~ \beta_{i-1},
  \label{eq:cvx_anneal}
\end{equation}
using binary search. This adaptive scheme is
crucial in our problem setting, where we \emph{a priori} have
no knowledge of appropriate scales for $f$ and, as a result, $\beta$. In
practice, we find that forcing $\beta_i$ to monotonically increase in
$t$ yields better mixing, so we set
$\beta_{i}(t) = \max (\beta_{i}(t-1), \hat \beta_{i}(t))$, where $\hat
\beta_{i}(t)$  solves problem \eqref{eq:cvx_anneal}.

 \begin{figure}[!!t]
 	\centering
 	\ifarxiv
 	\includegraphics[width=0.65\columnwidth]{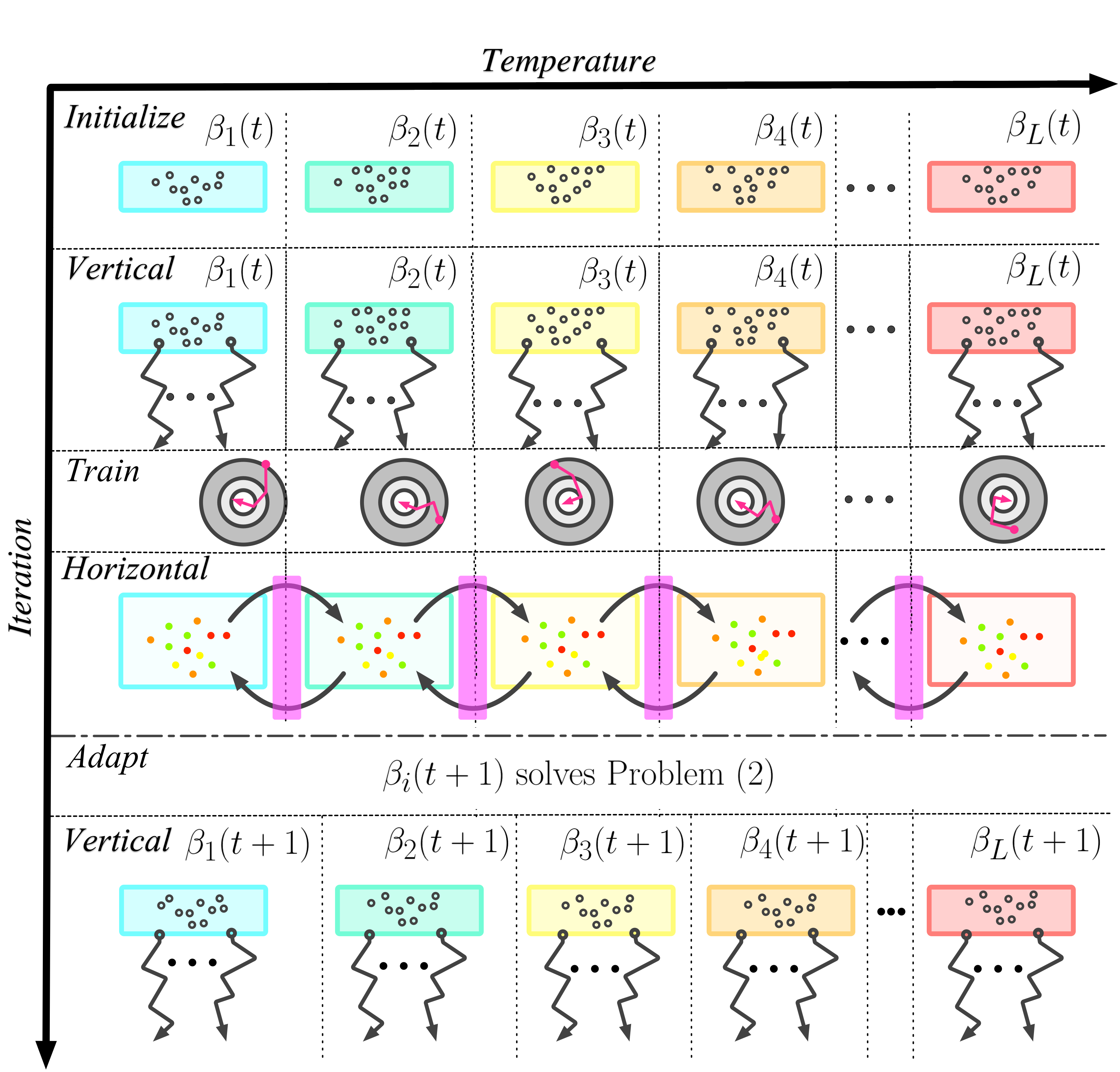}
 	\else
 	\vspace{-9pt}
 	\includegraphics[width=0.85\columnwidth]{figs/aadapt.png}
 	\vspace{-10pt}
 	\fi
 	\caption{Illustration of \textsc{AAdaPT}. Vertical MCMC steps (jagged black arrows) occur in parallel for $x^{i,j}$, followed by gradient descent for trainable parameters $\theta^{i,j}$ (magenta arrows) and horizontal MCMC configuration swaps between populations (curved black arrows). Temperatures $\beta_i(t)$ are then updated by problem \eqref{eq:cvx_anneal}.}
 	\label{fig:aadapt}
 	\ifarxiv
	\else
 	\vspace{-16pt}
 	\fi
 \end{figure}
 
\ifarxiv
\paragraph{Evaluating proposals via self-play:}
\else
\vspace{-2pt}
\textbf{Evaluating proposals via self-play:}
\fi
We apply \textsc{AAdaPT} to a multi-agent game. It is only possible to evaluate $f(x,\theta)$ in the context of other agents, but we consider the setting where demonstrations from potential opponents are either difficult to obtain or held secret. Thus, we iteratively evaluate $f$ via self-play. For each configuration $(x, \theta)$, we perform a race in the simulated environment between two vehicles with the same policy (with $f(x,\theta)$ being the lap time of the agent that starts behind the other). Vertical MCMC steps propose new $x$, which are then accepted according to MH criteria. After a number of vertical iterations, a stochastic gradient descent (SGD) step is applied to $\theta$ (which maximizes the likelihood of the trajectories chosen by the agent with cost functions parametrized by $x$). Following this process, the updated agents in adjacent temperature levels are exchanged via horizontal MCMC steps. Although we choose $f(x, \theta)$ as the laptime, explicit entropic terms can also be included to further encourage diversity within a single vertical chain or across the population.

At the conclusion of \textsc{AAdaPT}, we use the coldest population of $D$
agents at inverse temperature $\beta_1(T)$ to build a DPP
sampler. Specifically, define the matrix $H$
via configurations $x^{1,\cdot}$ at the lowest temperature
\begin{equation}\label{eq:maha_matrix}
H_{ab} = \| x^{1,a} - x^{1,b}\|.
\end{equation}
Then we define the DPP kernel $K$ as $K_{ab} = \exp\left ({-H_{ab}^2/\sigma^2} \right )$ with a scale parameter $\sigma=0.5$, and we sample $d \le D$ configurations from this DPP.

\section{Online learning with computation budgets}
\label{sec:bandits}
Now we exploit the population of $d$ learned prototype behaviors to enable
robust performance. The agent's (our) goal is to act robustly
 against uncertainty in opponent behaviors and adapt online to a given
opponent. We parametrize the agent's (stochastic) policy as follows. At each
time step, we sample goal states (consisting of pose and velocity) via
a generative model $G(\theta)$ parametrized by $\theta$ (as in
Section~\ref{sec:pbt}). For a given goal state, we compute the parameters of
a cubic spline that reaches the goal by solving a nonconvex trajectory
optimization problem~\cite{mcnaughton2011parallel}; on this proposed
trajectory we evaluate a collection of cost functions (such as the maximum
curvature or minimum acceleration along the path) weighted by
the vector $x$ (recall Section~\ref{sec:pbt}), similar
to \citet{sadat2019jointly} (see Appendix \ref{sec:software} for a 
description of all costs). Finally, we choose the sampled goal trajectory
with minimum robust cost and perform an action to track this trajectory.

Some of the costs that evaluate the utility of a goal state involve
beliefs about the opponent's future trajectory. For a goal $p$, 
we rewrite the performance objective at time $t$ with respect to
a protoype opponent $i$ as a receding-horizon cost
\begin{equation*}
c_i(t; p):=-\sum_{s>t} \lambda^{s-t} \E[r(o(s); p)],
\end{equation*}
where we omit dependence on the agent's cost weights $x$ for convenience. We
parametrize the agent's belief of the opponent's behavior as a categorical
distribution of beliefs over the prototypes. Specifically, let $w(t) \in
\Delta$ be a weight vector at a given time $t$, where $\Delta :=\{a \in
\R_+^d \mid a^T\mathbf{1}=1 \}$, and let
$P_0(t):=\operatorname{Categorical}(w(t))$. Then $P_0(t)$ is the nominal
distribution describing the agent's belief about the opponent. Furthermore,
we consider ambiguity sets $\mc{P}(t)$ defined by divergence measures on the
space of probability measures over $\Delta$. For a convex function $\phi$
with $\phi(1) = 0$, the $\phi$-divergence between distributions $P$ and $Q$
is $\fdiv{P}{Q} = \int \phi(\frac{dP}{dQ}) dQ$. We use sets $\mc{P}(t)
\defeq \left\{Q : \fdiv{Q}{P_0}(t) \le \tol \right\}$ where $\tol > 0$ is a
specified constant.  Our implementation employs the $\chi^2$-divergence
$\phi(t) = t^2 - 1$.

Having defined the ambiguity set $\mc{P}(t)$ and the cost with respect to each prototype opponent, we rewrite the robust performance objective  \eqref{eq:robrl} to clearly illustrate the optimization problem. Let $C(t;p)$ be a random variable representing the expected cost with respect to the belief of the opponent (and goal state $p$). Then the robust cost at time $t$ is
\begin{equation} \label{eq:robustcost}
  \sup_{Q \in \mc{P}(t)} \!\! \E_Q[C(t;p)]
  = \sup_{q: \sum_i w_i \phi(\frac{q_i}{w_i}) \le \tol} \sum_i q_ic_i(t;p).
\end{equation}
When $\rho=0$, this is the expected cost under $P_0$; larger $\rho$ adds robustness. Solving the convex optimization problem \eqref{eq:robustcost} first requires computing the costs $c_i(t)$. Using $\lambda \ge 0$ for the constraint $\fdiv{Q}{P_0} \le
\tol$, a partial Lagrangian is
\begin{equation*}
  \mc{L}(q, \lambda) = \sum_i q_i c_i(t)
  - \lambda \left (\sum_i w_i \phi\left ({q_i}/{w_i} \right ) - \tol \right ).
\end{equation*}
The corresponding dual function is $v(\lambda)=\sup_{q \in \Delta}\mc{L}({q, \lambda})$, and minimizing $v(\lambda)$ via bisection yields the solution to problem~\eqref{eq:robustcost}. Maximizing $\mc{L}(q, \lambda)$ with respect to $q$ for a given $\lambda$ requires $O(d)$ time using a variant of median-based search \cite{Duchi_2008} (see Appendix~\ref{sec:online_appendix}). Thus, computing an $\epsilon$-suboptimal solution uses $O(d\log(1/\epsilon))$ time.

The supremum in the robust cost~\eqref{eq:robustcost} is over belief ambiguity. Thus, our approach generalizes beyond the goal-sampling and trajectory-optimization approach presented at the beginning of this section; it is compatible with any policy that minimizes a cost $c_i(t)$ with respect to a parametrization for opponent $i$'s policy. In this way, it is straightforward to combine our framework with robust model predictive control formulations that have rigorous stability guarantees.

In order to perform competitive actions, the agent updates the ambiguity set $\mc{P}(t)$ and computes the robust cost \eqref{eq:robustcost} on an embedded processor on board the vehicle in real-time (\eg~within 100 milliseconds). In the next two subsections, we describe how to perform both operations in the presence of a severely limited computational budget, and we quantitatively analyze the implications of the budget on the robustness/performance tradeoff.

\subsection{Approximating the robust cost}\label{sec:approx_cost}

For a large library of prototypical opponents (large $d$), computing every $c_i$ in the objective  \eqref{eq:robustcost} is prohibitively expensive. Instead, we consider an empirical approximation of the objective, where we draw $N_w$ indices $J_k \simiid P_0(t)$ (where $N_w<d$) and consider the weighted sum of these costs $c_{j_k}$. Specifically, we define the empirical approximation $\mc{P}_{N_w}:=\{q:\fdiv{q}{\mathbf{1}/N_w} \le \tol\}$ to $\mc{P}$ and solve the following empirical version of problem~\eqref{eq:robustcost}:
\begin{equation}\label{eq:real-problem}
  \maximize_{q \in \mc{P}_{N_w}} ~~ \sum_k q_k c_{j_k}(t;p).
\end{equation}

This optimization problem \eqref{eq:real-problem} makes manifest the price of robustness in two ways. The first involves the setup of the problem---computing the $c_{j_k}$. First, we denote the empirical distribution as $\hat w(t)$ with $\hat w_i(t)= \sum_k^{N_w}\mathbf{1}\{j_k = i\}/N_w$. Even for relatively small $N_w/d$, $\hat w(t)$ concentrates closely around $w(t)$ (see \eg~\citet{weissman2003inequalities} for a high-probability bound).
Thus, when the vehicle's belief about its opponent $w(t)$ is nearly uniform, the $j_k$ values have few repeats. Conversely, when the belief is peaked at a few opponents, the number of unique indices is much smaller than $N_w$, allowing faster computation of $c_{j_k}$. The short setup-time enables faster planning or, alternatively, the ability to compute the costs $c_{j_k}$ with longer horizons. Therefore, theoretical performance automatically improves as the vehicle learns about the opponent and the robust evaluation approaches the true cost.

The second way we illustrate the price of robustness is by quantifying the quality of the approximation \eqref{eq:real-problem} with respect to the number of samples $N_w$. For shorthand, define the true expected and approximate expected costs for goal $p$ and distributions $Q$ and $q$ respectively as
\begin{equation*}
  R(Q;p):=\E_Q[C(t;p)], \;\; \hat R(q;p) := \frac{1}{N_w}\sum_{k=1}^{N_w}q_kc_{j_k}(t;p).
\end{equation*}
Then, we have the following bound:
\begin{proposition}[Approximation quality]
  \label{prop:approx}
  Suppose $C(t; p)\in [-1,1]$ for all $t, p$. Let $A_\tol =
  \frac{2(\tol + 1)}{\sqrt{1 + \tol} - 1}$ and $B_{\tol}=\sqrt{8(1+\tol)}$.
  Then with probability at least $1 - \delta$
  over the $N_w$ samples $J_k \simiid P_0$,
  \begin{small}
  \begin{equation*}
    \bigg| \sup_{q\in \mc{P}_{N_w}}\hat R(q;p) - \sup_{Q\in \mc{P}} R(Q;p) \bigg | \le 4A_{\tol}\sqrt{\frac{\log(2N_w)}{N_w}}+B_{\tol}\sqrt\frac{{\log \frac{2}{\delta}}}{{N_w}}
  \end{equation*}
  \end{small}
\end{proposition}
\noindent See Appendix \ref{sec:online_appendix} for the proof. Intuitively, increasing accuracy of the robust cost requires more samples (larger $N_w$), which comes at the expense of computation time. Similar to computing the full cost \eqref{eq:robustcost}, $\epsilon$-optimal solutions require $O(N_u\log(1/\epsilon))$ time for $N_u \le N_w$ unique indices $j_k$. In our experiments (\cf~Section~\ref{sec:experiments}), most of the computation time involves the setup to compute the $N_u$ costs $c_{j_k}$.

\subsection{Updating the ambiguity set}\label{sec:bandit}
To maximize performance against an opponent, the agent updates the ambiguity set $\mc P$ as the race progresses. Since we consider $\phi$-divergence balls of fixed size $\tol$, this update involves only the nominal belief vector $w(t)$. As with computation of the robust cost, this update must occur efficiently due to time and computational constraints.

For a given sequence of observations of the opponent $o^H_{\rm opp}(t):=\{o_{\rm opp}(t), o_{\rm opp}(t-1),...,o_{\rm opp}(t-h+1)\}$ over a horizon $h$, we define the likelihood of this sequence coming from the $i$\textsuperscript{th} prototype opponent as
\begin{equation}
l^h_i(t) = \log d\P\left (o^h_{\rm opp}(t) | G(\theta^{1,i}) \right ),
\end{equation}
where $G(\theta^{1,i})$ is a generative model of goal states for the $i$\textsuperscript{th} prototype opponent. 
Letting $\bar{l}$ be a uniform upper bound on $l^h_i(t)$, we define the losses $L_i(t):=1-l^h_i(t)/\bar{l}$.

If we had enough time/computation budget, we could compute $L_i(t)$ for all prototype opponents $i$ and perform an online mirror descent update with an entropic Bregman divergence \cite{shalev2012online}. In a resource-constrained setting, we can only select a few of these losses, so we use EXP3~\cite{auer2002nonstochastic} to update $w(t)$. Unlike a standard adversarial bandit setting, where we pull just one arm (\eg compute a loss $L_i(t)$) at every time step, we may have resources to compute up to $N_w$ losses in parallel at any given time (the same indices $J_k$ discussed in Section \ref{sec:approx_cost}). Denote our unbiased subgradient estimate as $\gamma(t)$:
\begin{equation}\label{eq:exp3grad}
\gamma_{i}(t) = \frac{1}{N_w}\sum_{k=1}^{N_w}\frac{L_i(t)}{w_i(t)}\mathbf{1}\{J_k = i \}.
\end{equation}
Algorithm~\ref{alg:exp3} describes our slightly modified EXP3 algorithm, which has the following expected regret.

\ifarxiv
\begin{algorithm}[tb]
	\caption{EXP3 with $N_w$ arm-pulls per iteration}
	\label{alg:exp3}
	\begin{small}
	\begin{algorithmic}
		\State {\bfseries Input:} Stepsize sequence $\eta_t$, $w(0):= \mathbf{1}/d$, steps $T$
		\State{{\bfseries for }$t=0$ {\bfseries to} $T-1$}
		\State{$\;\;\;\;$ Sample $N_w$ indices $J_k\simiid \operatorname{Categorical}(w(t))$}
		\State{$\;\;\;\;$ Compute $\gamma(t)$ (Equation \eqref{eq:exp3grad})}
		\State{$\;\;\;\;$ $w_i(t+1):= \frac{w_i(t) \exp \left ({-\eta_t\gamma_i(t)} \right )}{\sum_{j=1}^d w_j(t) \exp \left ({-\eta_t\gamma_j(t)} \right ) }$}
	\end{algorithmic}
	    \end{small}
\end{algorithm}
\else
\begin{algorithm}[tb]
	\caption{EXP3 with $N_w$ arm-pulls per iteration}
	\label{alg:exp3}
	\begin{small}
	\begin{algorithmic}
		\STATE {\bfseries Input:} Stepsize sequence $\eta_t$, $w(0):= \mathbf{1}/d$, steps $T$
		\STATE{{\bfseries for }$t=0$ {\bfseries to} $T-1$}
		\STATE{$\;\;\;\;$ Sample $N_w$ indices $J_k\simiid \operatorname{Categorical}(w(t))$}
		\STATE{$\;\;\;\;$ Compute $\gamma(t)$ (Equation \eqref{eq:exp3grad})}
		\STATE{$\;\;\;\;$ $w_i(t+1):= \frac{w_i(t) \exp \left ({-\eta_t\gamma_i(t)} \right )}{\sum_{j=1}^d w_j(t) \exp \left ({-\eta_t\gamma_j(t)} \right ) }$}
	\end{algorithmic}
	    \end{small}
\end{algorithm}
\fi

\begin{proposition}\label{prop:exp3}
Let $z := \frac{d-1}{N_w}+1$. Algorithm~\ref{alg:exp3} run for $T$ iterations with stepsize $\eta=\sqrt{\frac{2\log (d)}{zT}}$ has expected regret bounded by $\sum_{t=1}^T \E \left [ \gamma(t)^T(w(t) - w^{\star}) \right ] \le \sqrt{2z T\log (d)}$.
\end{proposition}
\noindent See Appendix \ref{sec:online_appendix} for the proof. This regret bound looks similar to that if we simply ran $N_w$ standard EXP3 steps per iteration $t$ (in which case $z= d/N_w$). However, our approach enables parallel computation which is critical in our time-constrained setting. Note that the ``multiple-play'' setting we propose here has been studied before with better regret bounds but higher computational complexity per iteration \cite{uchiya2010algorithms, zhou2018budget}. We prefer our approach for its simplicity and ability to be easily combined with the robust-cost computation.

\section{Experiments}
\label{sec:experiments}
In this section we first describe the AR environment used to conduct our experiments. 
Next we explore the hyperparameters of the algorithms in Section~\ref{sec:pbt} and~\ref{sec:bandits}, identifying a preferred configuration. 
Then we consider the overarching hypothesis: online adaptation can improve the performance of robust control strategies. We show the statistically significant results affirming the theory and validate the approach's performance on real vehicles. 

The experiments use an existing low-cost 1/10$^{th}$-scale, Ackermann-steered AV (Figure~\ref{fig:carmain}). %
Additionally, we create a simulator and an associated OpenAI Gym API \cite{gym} suitable for distributed computing. The simulator supports multiple agents as well as deterministic executions. We experimentally determine the physical parameters of the agent models for simulation and use SLAM to build the virtual track as a mirror of a real location (see Figure~\ref{fig:multimodal}). The hardware specifications, software, and simulator are open-source~\footnote{\url{https://github.com/travelbureau/f0_icml_code}} (see Appendices~\ref{sec:hardware} and \ref{sec:software} for details).

The agent software uses a hierarchical planner \cite{gat1998three} similar to \citet{ferguson2008motion}. The key difference is the use of a masked autoregressive flow (MAF)~\cite{rezende2015variational} which provides the generative model for goal states, $G(\theta)$. Belief inference and robust cost computation require sampling and evaluating the likelihood of goal states. MAFs can evaluate likelihoods quickly but generate samples slowly. Inspired by \citet{oord2018parallel} we overcome this inefficency by training a ``student'' inverse autogressive flow (IAF)~\cite{kingma2016improved} on MAF samples. Given a sample of goals from the IAF, the agent synthesizes dynamically feasible trajectories following \citet{mcnaughton2011parallel}. Each sample is evaluated according to Equation \ref{eq:robustcost}; the weights of the cost functions are learned by \textsc{AAdaPT} (and formal definitions of the cost components are in Appendix~\ref{sec:software}). Belief updates use Algorithm \ref{alg:exp3} using the MAF to compute the losses $L_i(t)$. 

\ifarxiv
\begin{figure}[!!!t]
	\centering
	\includegraphics[width=0.5\columnwidth]{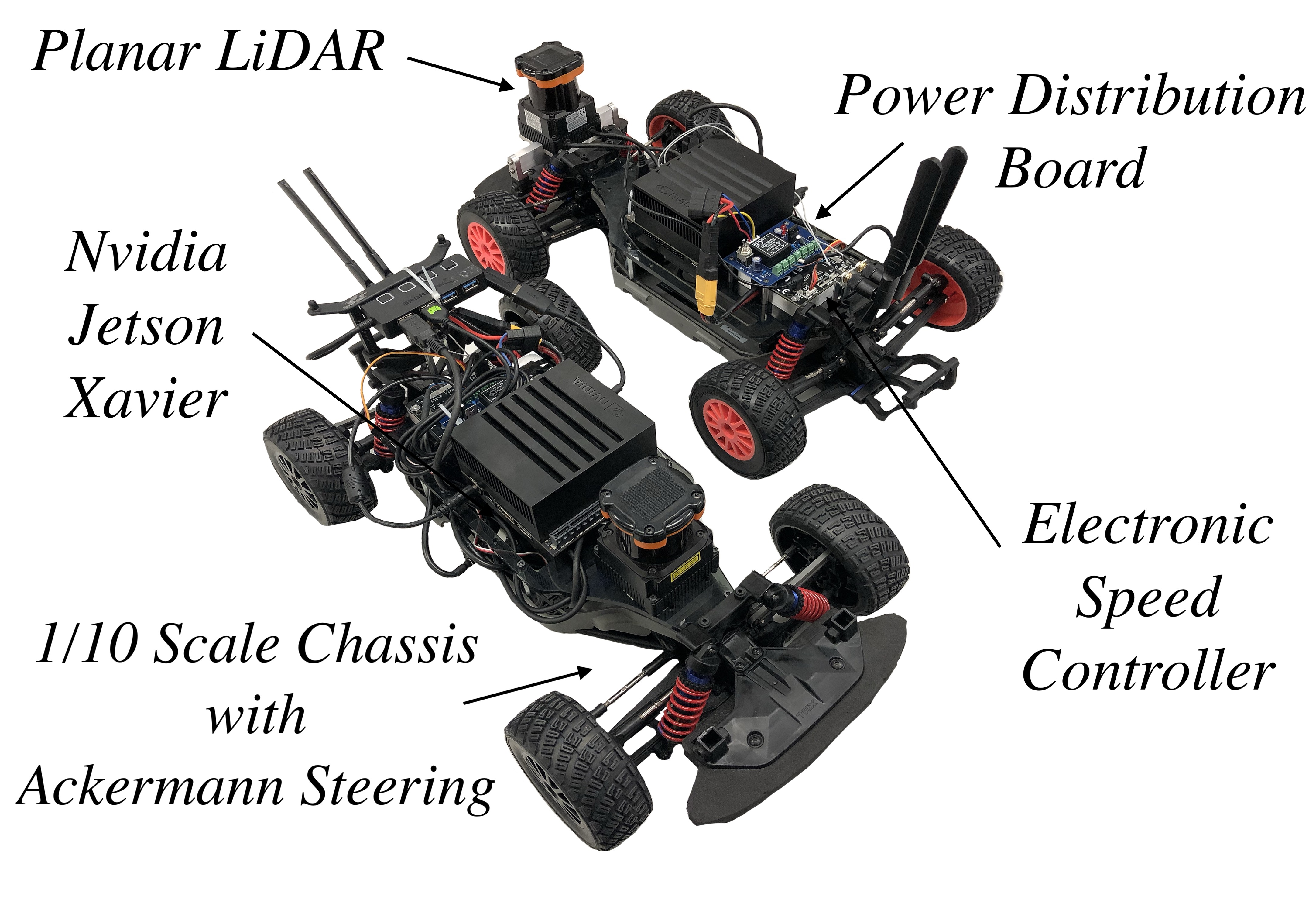}
	\ifarxiv
	\else
	\vspace{-15pt}
	\fi
	\caption{Components of the 1/10-scale vehicle}
	\label{fig:carmain}
\end{figure}
\begin{figure}[!!t]
	\begin{minipage}{0.5\columnwidth}
		\centering
		\subfigure[Performance vs. iteration]{\includegraphics[width=1.0\textwidth]{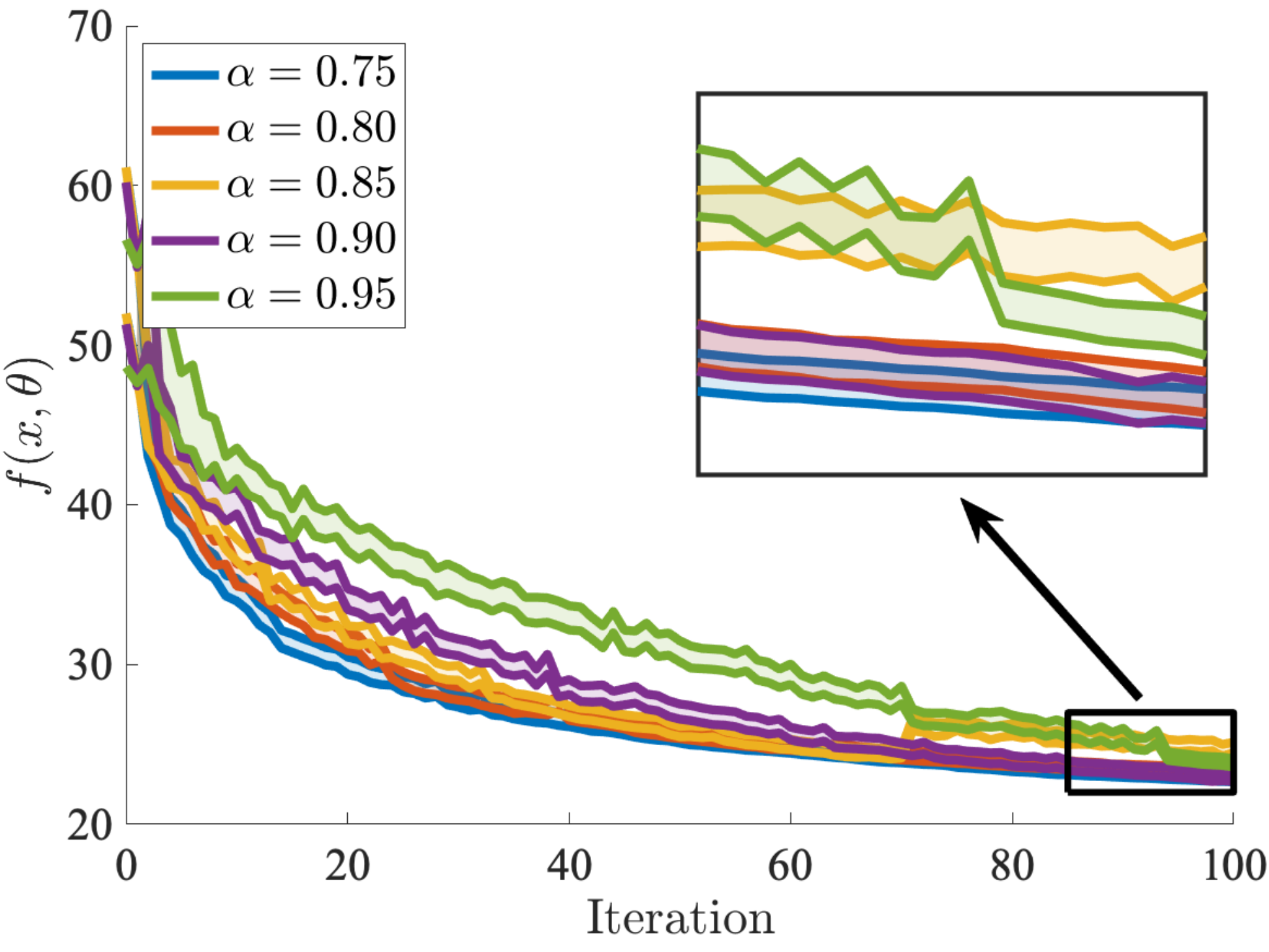}}%
	\end{minipage}\begin{minipage}{0.5\columnwidth}
		\centering
		\subfigure[Diversity vs. iteration]{\includegraphics[width=1.0\textwidth]{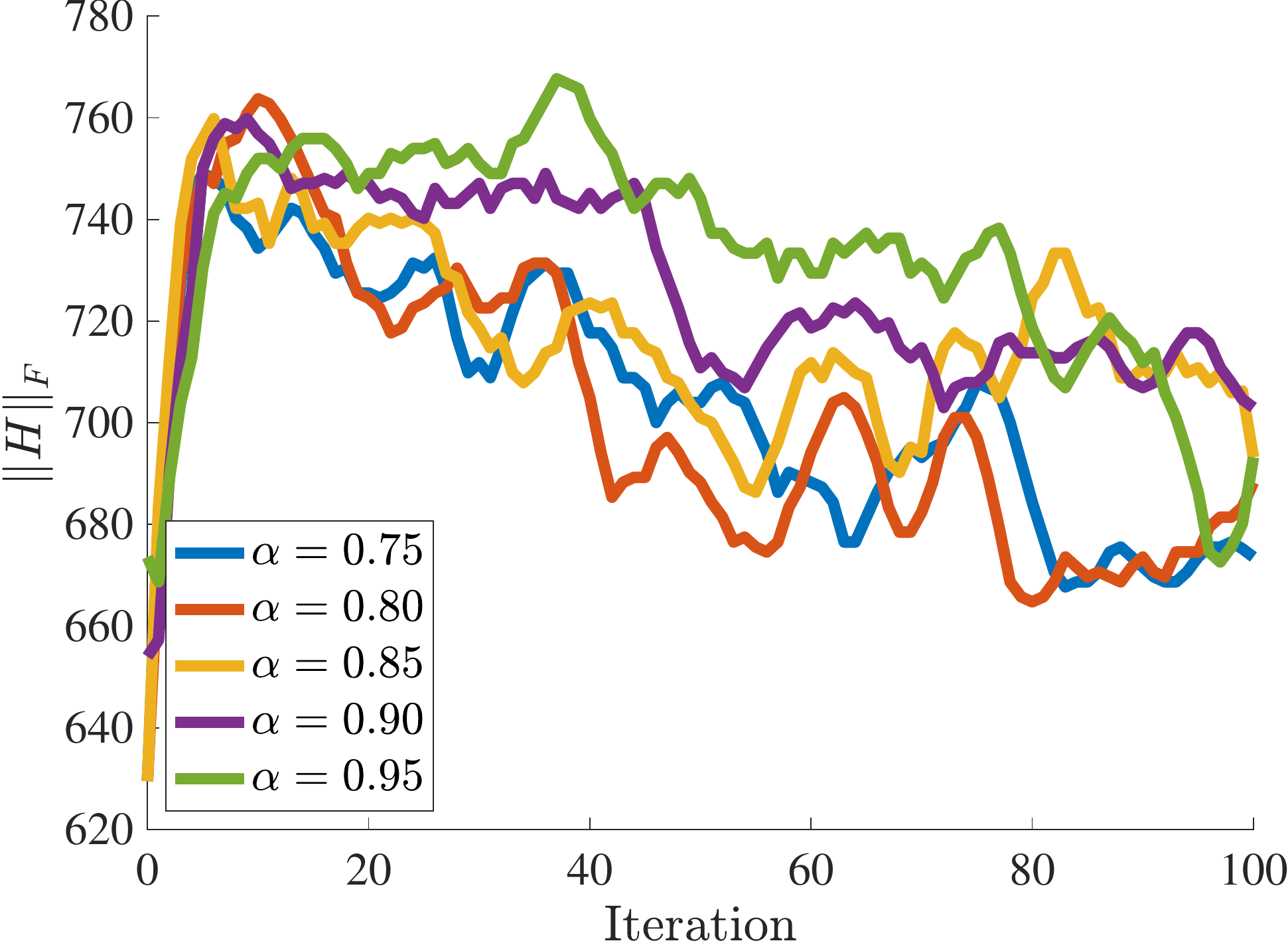}}%
	\end{minipage}
	\caption[]{\label{fig:temper} Hyperparameter selection for \textsc{AAdaPT}. (a) 95\%-confidence intervals for $f(x,\theta)$ in the coldest temperature level. (b) Frobenius norm of the Mahalanobis distance matrix $H$ \eqref{eq:maha_matrix}. The value $\alpha=0.9$ achieves the best performance and diversity.}
\end{figure}
\else
\begin{figure}[!!!t]
	\centering
	\includegraphics[width=0.85\columnwidth]{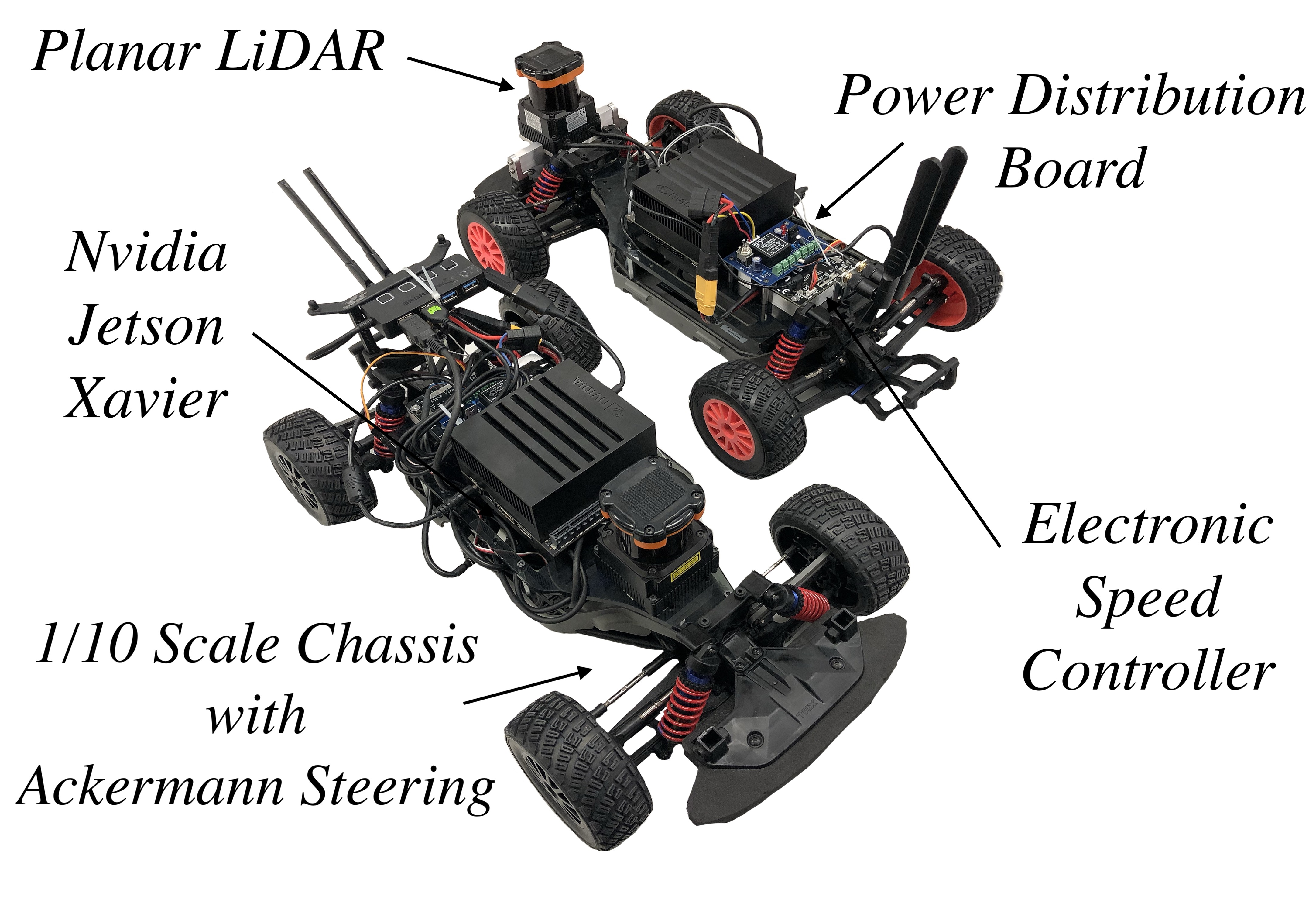}
	\vspace{-15pt}
	\caption{Components of the 1/10-scale vehicle}
	\label{fig:carmain}
	\vskip -10pt
\end{figure}
\fi

\subsection{Offline population synthesis}
We run \textsc{AAdaPT} with $L=5$ populations, $D=160$ configurations per population, and $T=100$ iterations. For vertical MCMC steps, we randomly sample $16$ configuratons per population and perform $V=2$ iterations of 5 hit-and-run proposals. Furthermore, we perform $E=DL^2/\alpha^{t/(L-1)}$ horizontal steps (motivated by the fact fact that ``tunneling'' from the highest-temperature level to the coldest takes $O(L^2)$ accepted steps). Finally, for training $\theta$, we use Adam \cite{kingma2014adam} with a learning rate of $10^{-4}$.

Figure \ref{fig:temper} shows results with 5 choices for the most influential hyperparameter, the annealing rate: $\alpha\in \{0.75, 0.80, 0.85, 0.90, 0.95\}$. Figure \ref{fig:temper}(a) displays $95\%$-confidence intervals for the mean laptime in the coldest level. The annealing rates $\alpha \in \{0.75, 0.80, 0.90\}$ all result in comparable performance of $22.95 \pm 0.14$ (mean $\pm$ standard error) seconds at the end of the two-lap run. Figure \ref{fig:temper}(b) illustrates a metric for measuring diversity, the Frobenius norm of the Mahalanobis distance matrix \eqref{eq:maha_matrix}. We see that $\alpha=0.9$ results in the highest diversity while also attaining the best performance. Thus, in further experimentation, we use the results from the run conducted with $\alpha=0.9$.

Figure~\ref{fig:multimodal} illustrates qualitative differences between cost functions. Figure \ref{fig:multimodal}(a) displays trajectories for agents driven using 5 cost functions sampled from the learned DPP. The cornering behavior is quite different between the trajectories. Figure \ref{fig:multimodal}(b) displays the trajectories chosen by all 160 agents in the population at $\beta_1(T)$ at various snapshots along the track. There is a wider spread of behavior near turns than areas where the car simply drives straight.

\ifarxiv
\else
\begin{figure}[!!t]
	\begin{minipage}{0.49\columnwidth}
		\centering
		\subfigure[Performance vs. iteration]{\includegraphics[width=1.0\textwidth]{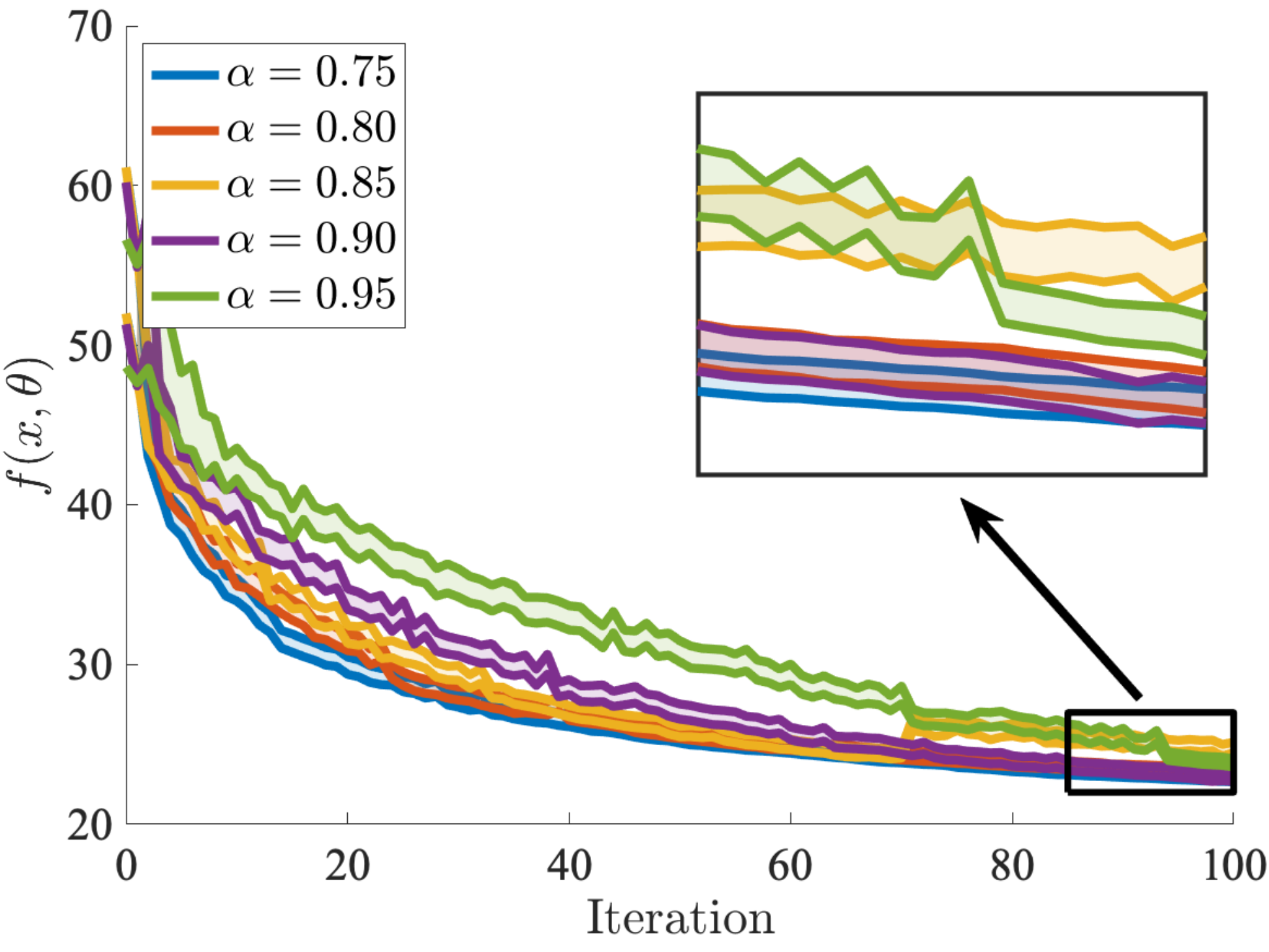}}%
	\end{minipage}
	\centering
	\begin{minipage}{0.49\columnwidth}
		\centering
		\subfigure[Diversity vs. iteration]{\includegraphics[width=1.0\textwidth]{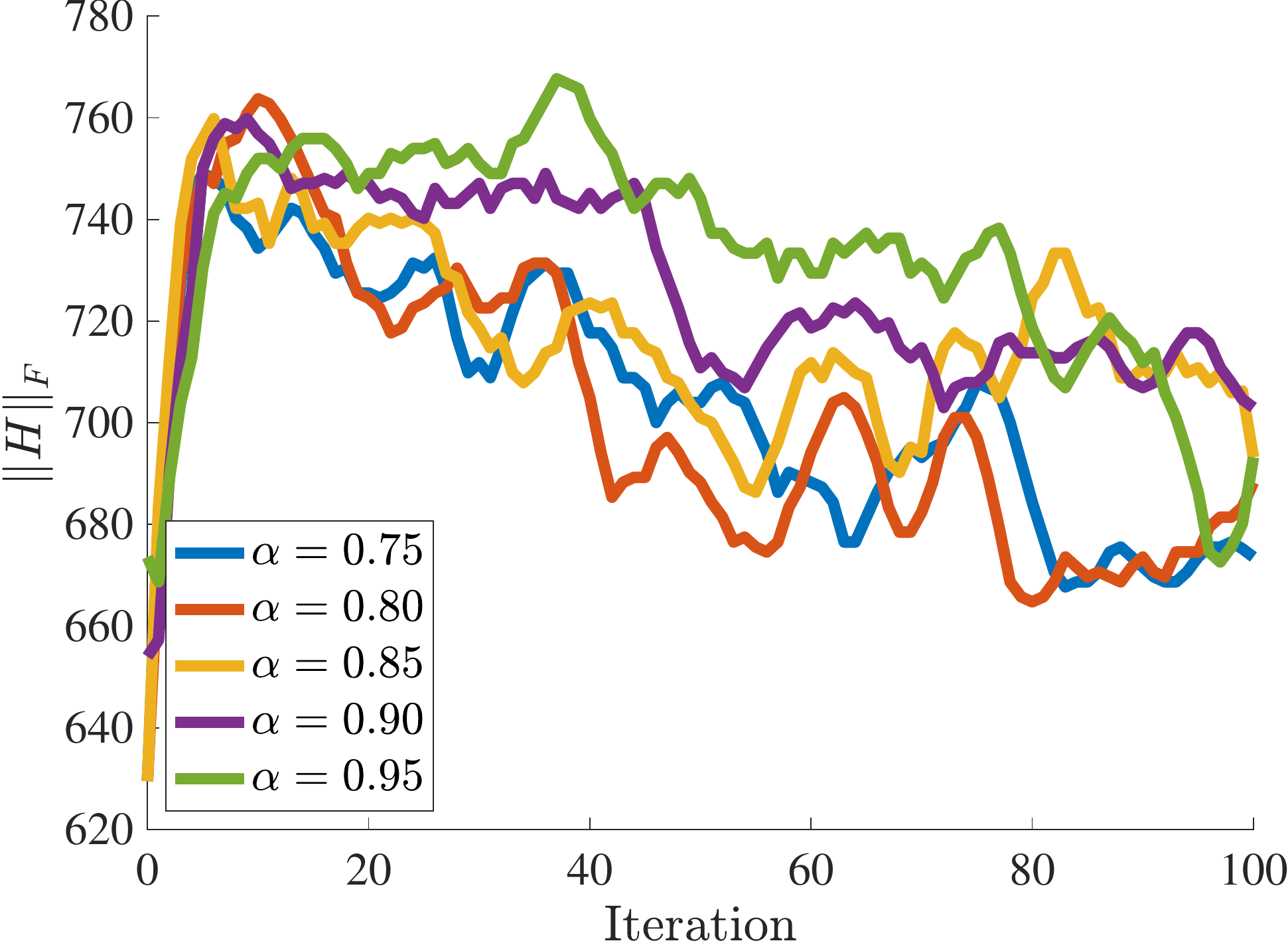}}%
	\end{minipage}
	\vspace{-10pt}
	\caption[]{\label{fig:temper} Hyperparameter selection for \textsc{AAdaPT}. (a) 95\%-confidence intervals for $f(x,\theta)$ in the coldest temperature level. (b) Frobenius norm of the Mahalanobis distance matrix $H$ \eqref{eq:maha_matrix}. The value $\alpha=0.9$ achieves the best performance and diversity.}
	\vspace{-5pt}
\end{figure}
\fi

\begin{figure}[!!t]
	\begin{minipage}{0.67\columnwidth}
		\centering
		\subfigure[Rollouts from 5 agents]{\includegraphics[width=1.0\textwidth]{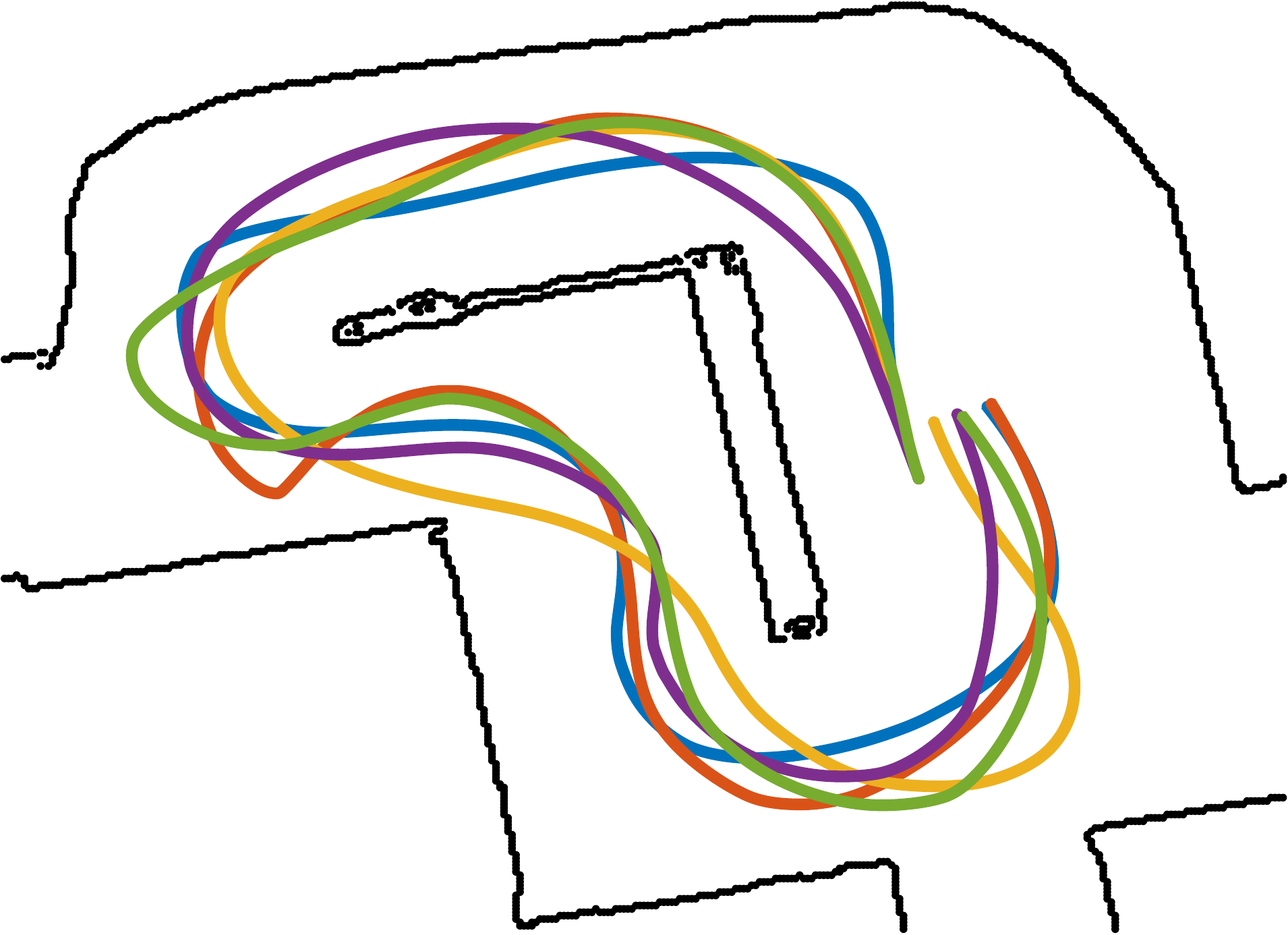}}%
	\end{minipage}
	\centering
	\begin{minipage}{0.67\columnwidth}
		\centering
		\subfigure[Snapshot trajectories]{\includegraphics[width=1.0\textwidth]{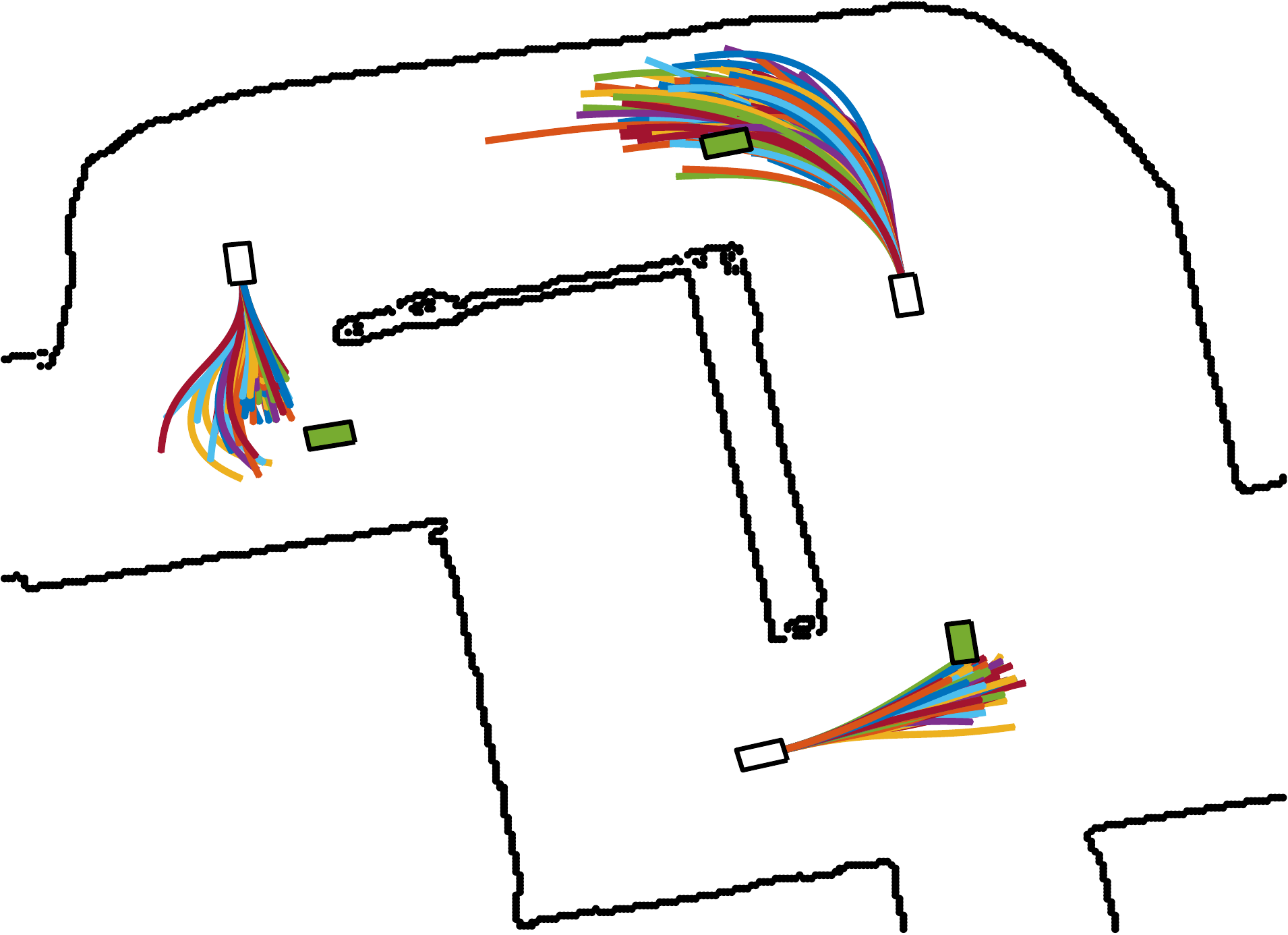}}%
	\end{minipage}
	\vspace{-10pt}
	\caption[]{\label{fig:multimodal} Qualitative illustrations of multimodal behavior in the learned population of cost functions}
	\ifarxiv
	\else
	\vspace{-10pt}
	\fi
\end{figure}

\subsection{Simulated experiments}
\ifarxiv
We conduct a series of tests in simulation to determine the effects of distributional robustness and adaptivity on overall safety and performance. For a given robustness level $\rho/N_w \in \{0.001, 0.025, 0.2,\\ 0.4, 0.75, 1.0\}$ (with $N_w=8$ for all experiments), we simulate 40 two-lap races against each of the $d=10$ diverse opponents sampled from the DPP. For fair comparisons, half of the races have the opponent starting on the outside and the other half with the opponent on the inside of the track. Importantly, these experiments involve only the most elite policies from the temperature level $\beta_1(T)$. Since the physical characteristics of the vehicles are identical, win rates between elite policies significantly greater than $0.5$ are meaningful. In contrast, against a set of weaker opponents sampled via DPP from the 3\textsuperscript{rd} temperature level $\beta_3(T)$, the win-rate (fraction of races that our agent from the coldest temperature wins) is $0.848 \pm 0.012$.
\else
We conduct a series of tests in simulation to determine the effects of distributional robustness and adaptivity on overall safety and performance. For a given robustness level $\rho/N_w \in \{0.001, 0.025, 0.2, 0.4, 0.75, 1.0\}$ (with $N_w=8$ for all experiments), we simulate 40 two-lap races against each of the $d=10$ diverse opponents sampled from the DPP. For fair comparisons, half of the races have the opponent starting on the outside and the other half with the opponent on the inside of the track. Importantly, these experiments involve only the most elite policies from the temperature level $\beta_1(T)$. Since the physical characteristics of the vehicles are identical, win rates between elite policies significantly greater than $0.5$ are meaningful. In contrast, against a set of weaker opponents sampled via DPP from the 3\textsuperscript{rd} temperature level $\beta_3(T)$, the win-rate (fraction of races that our agent from the coldest temperature wins) is $0.848 \pm 0.012$.
\vspace{-5pt}
\fi

\paragraph{Effects of distributional robustness} We test the hypothesis that distributional robustness results in more conservative policies. For every race both agents have a fixed robustness level $\rho$ and no adaptivity. To measure aggressiveness/conservativeness, we consider instantaneous time-to-collision (iTTC) of the vehicles during the race (see Appendix \ref{sec:experiments_appendix}). Smaller iTTC values imply more dangerous scenarios and more aggressive policies. In Table \ref{table:robust}, we track the rate at which iTTC $< 0.5$ seconds. As expected, aggressiveness decreases with robustness (the rate of small iTTC values decreases as $\rho$ increases). The trend is $a+b\log(\rho)$, where $a=5.16\pm0.34$ and $b=-0.36\pm0.10$ ($R^2=0.75$).
\ifarxiv
\else
\vspace{-15pt}
\fi

\begin{table}
	\caption{The effect of distributional robustness on aggressiveness }
	\label{table:robust}
	\begin{center}
		\begin{small}
				\begin{tabular}{lc}
					\toprule
					Agent    & \% of iTTC values $< 0.5$s \\
					\midrule
                    $\rho/N_w=0.001$ & 7.86$\pm$ 0.90 \\
					  $\rho/N_w=0.025$ & 6.46$\pm$ 0.78  \\
					  $\rho/N_w=0.2$   & 4.75$\pm$ 0.65 \\
					  $\rho/N_w=0.4$  &  5.41$\pm$ 0.74   \\
					  $\rho/N_w=0.75$  & 5.50$\pm$ 0.82 \\
					  $\rho/N_w=1.0$   & 5.76$\pm$ 0.84  \\
					\bottomrule
				\end{tabular}
		\end{small}
	\end{center}
	\ifarxiv
	\else
	\vspace{-17pt}
	\fi
\end{table}
\begin{table}
	\caption{The effect of adaptivity on win-rate}
	\label{table:adaptivity}
	\begin{center}
		\begin{small}
				\begin{tabular}{lccc}
					\toprule
					& Win-rate & Win-rate\\
					Agent     & Non-adaptive     & Adaptive     & p-value \\
					\midrule
					$\rho/N_w=0.001$ & 0.593$\pm$ 0.025 & 0.588$\pm$ 0.025 &  0.84 \\
					$\rho/N_w=0.025$ & 0.593$\pm$ 0.025 & 0.600$\pm$ 0.024 &  0.77 \\
					$\rho/N_w=0.2$   & 0.538$\pm$ 0.025 & 0.588$\pm$ 0.025 &  0.045\\
					$\rho/N_w=0.4$  & 0.503$\pm$ 0.025 & 0.573$\pm$ 0.025 &  0.0098\\
					$\rho/N_w=0.75$  & 0.513$\pm$ 0.025 & 0.593$\pm$ 0.025 &  0.0013\\
					$\rho/N_w=1.0$   & 0.498$\pm$ 0.025 & 0.590$\pm$ 0.025 &  0.00024\\
					\bottomrule
				\end{tabular}
		\end{small}
	\end{center}
	\vskip -0.2in
\end{table}

\paragraph{Effects of adaptivity} Now we investigate the effects of online learning on the outcomes of races. Figure \ref{fig:regret}(a) shows that Algorithm~\ref{alg:exp3} identifies the opponent vehicle within approximately 150 timesteps (15 seconds), as illustrated by the settling of the regret curve.\footnote{We omit 3 of the regret lines for clarity in the plot.} Given evidence that the opponent model can be identified, we investigate whether adaptivity improves performance, as measured by win-rate. Table~\ref{table:adaptivity} displays results of paired t-tests for multiple robustness levels (with a null-hypothesis that adaptivity does not change the win-rate). Each test compares the effect of adaptivity for our agent on the 400 paired trials (and the opponents are always nonadaptive). Adaptivity significantly improves performance for the larger robustness levels $\rho/N_w \ge 0.2$. As hypothesized above, adaptivity automatically increases aggressiveness as the agent learns about its opponent and samples fewer of the other arms to compute the empirical robust cost \eqref{eq:real-problem}. This effect is more prominent when robustness levels are greater, where adaptivity brings the win-rate back to its level without robustness ($\rho/N_w=0.001$). Thus, the agent successfully balances safety and performance by combining distributional robustness with adaptivity.

\begin{figure}[!!t]
	\begin{minipage}{0.49\columnwidth}
		\centering
		\subfigure[Simulation]{\includegraphics[width=1.0\textwidth]{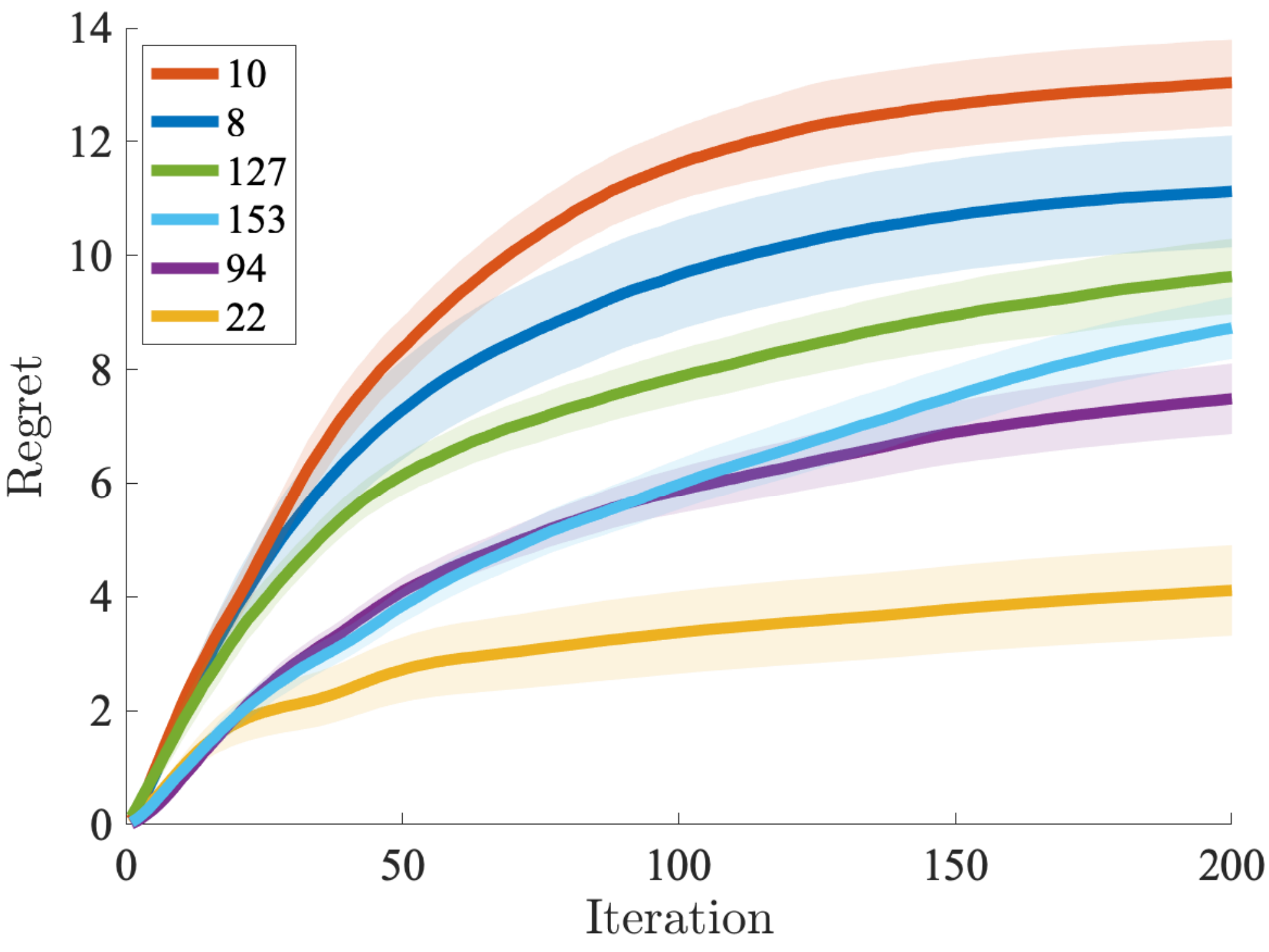}}%
	\end{minipage}
	\centering
	\begin{minipage}{0.49\columnwidth}
		\centering
		\subfigure[Real]{\includegraphics[width=1.0\textwidth]{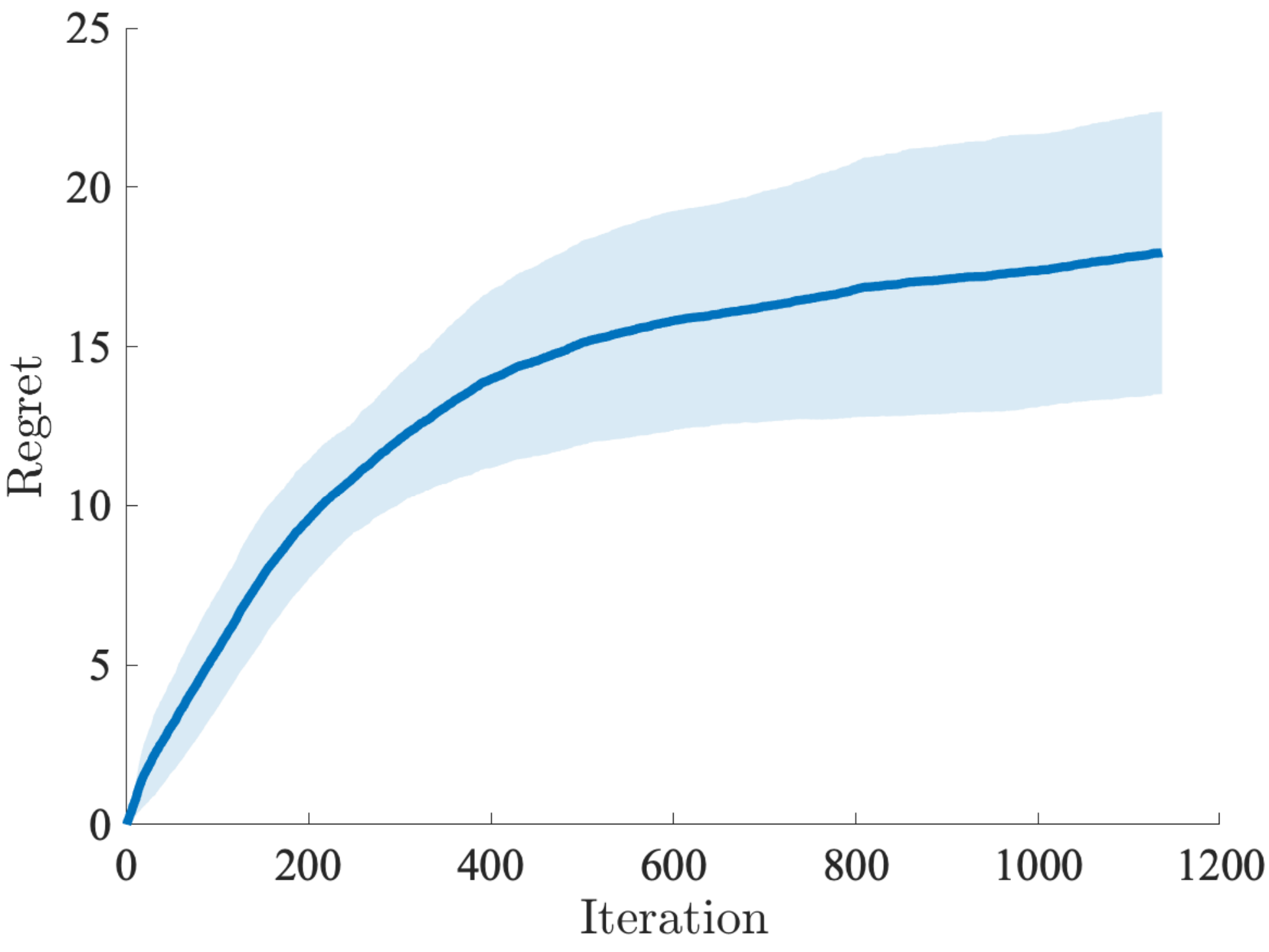}}%
	\end{minipage}
	\ifarxiv
	\else
	\fi
	\caption[]{\label{fig:regret} 95\%-confidence intervals for regret using $N_w=8$ arms in (a) simulation and (b) reality. The legend in (a) denotes opponent id and the opponent in (b) has id 22. Our agent has id 33.}
	\ifarxiv
	\else
	\vspace{-10pt}
	\fi
\end{figure}

\subsection{Real-world validation}
The real world experiments consist of races between agents 22 and 33; we examine the transfer of the opponent modeling approach from simulation to reality. In Figure~\ref{fig:regret}(b) we plot 33's cumulative regret; it takes roughly 4 times as many observations relative to simulation-based experiments to identify the opponent (agent 22). We demonstrate the qualitative properties of the experiments in a video of real rollouts synchronized with corresponding simulations.\footnote{\url{https://youtu.be/7Yat9FZzE4g}}  State estimation error and measurement noise drive the gap between simulated and real performance. First, both vehicle poses are estimated with a particle filter, whereas simulation uses ground-truth states. Since we infer beliefs about an opponent's policy based on a prediction of their actions at a given state, pose estimation error negatively impacts the accuracy of this inference. Second, the simulator only captures the geometry of the track; in reality glass and metal surfaces significantly affect the LIDAR range measurements, which in turn impact the MAF and IAF networks. The convergence of the cumulative regret in Figure~\ref{fig:regret}(b) reflects that, despite the simulation/reality gap, our simulation-trained approach transfers to the real world. Diminishing the effect of the simulation/reality gap is the subject of future work (see Appendix \ref{sec:simulator}).

\subsection{Approximation analysis}
Sampling $N_w$ indices $J_k \simiid P_0(t)$ allows us to quickly compute the approximate robust cost (Section \ref{sec:approx_cost}) and perform a bandit-style update to the ambiguity set (Section \ref{sec:bandit}). Now we analyze the time-accuracy tradeoff of performing this sampling approximation rather than using all $d$ prototypical opponents at every time step. Figure \ref{fig:approx}(a) shows the difference in regret for the same experiments as in Figure \ref{fig:regret}(a) if we perform full online mirror-descent updates. Denoting the simulations in Figure \ref{fig:regret}(a) as $S$ and those with the full mirror descent update as $M$, we compute difference as $\text{Regret}_S - \text{Regret}_M$. As expected, the difference is positive, since receiving the true gradient is better than the noisy estimate~\eqref{eq:exp3grad}. Similarly, Figure \ref{fig:approx}(b) shows the percent increase in cumulative planning time for the same pairs (sampling vs. full online mirror descent), where percent increase is given by $100(\text{Time}_M - \text{Time}_S)/\text{Time}_S$. As the agent learns who the opponent is, it draws many repeats in the $N_w$ arms, whereas the full mirror descent update always performs $d$ computations. As a result, the percent increase in cumulative iteration time approaches a contant of approximately $1.5\times$. All of these comparisons are done in simulation, where the agent is not constrained to perform actions in under 100 milliseconds. Performing a full mirror descent update is impossible on the real car, as it requires too much time. 

\begin{figure}[!!t]
	\begin{minipage}{0.49\columnwidth}
		\centering
		\subfigure[Difference in regret]{\includegraphics[width=1.0\textwidth]{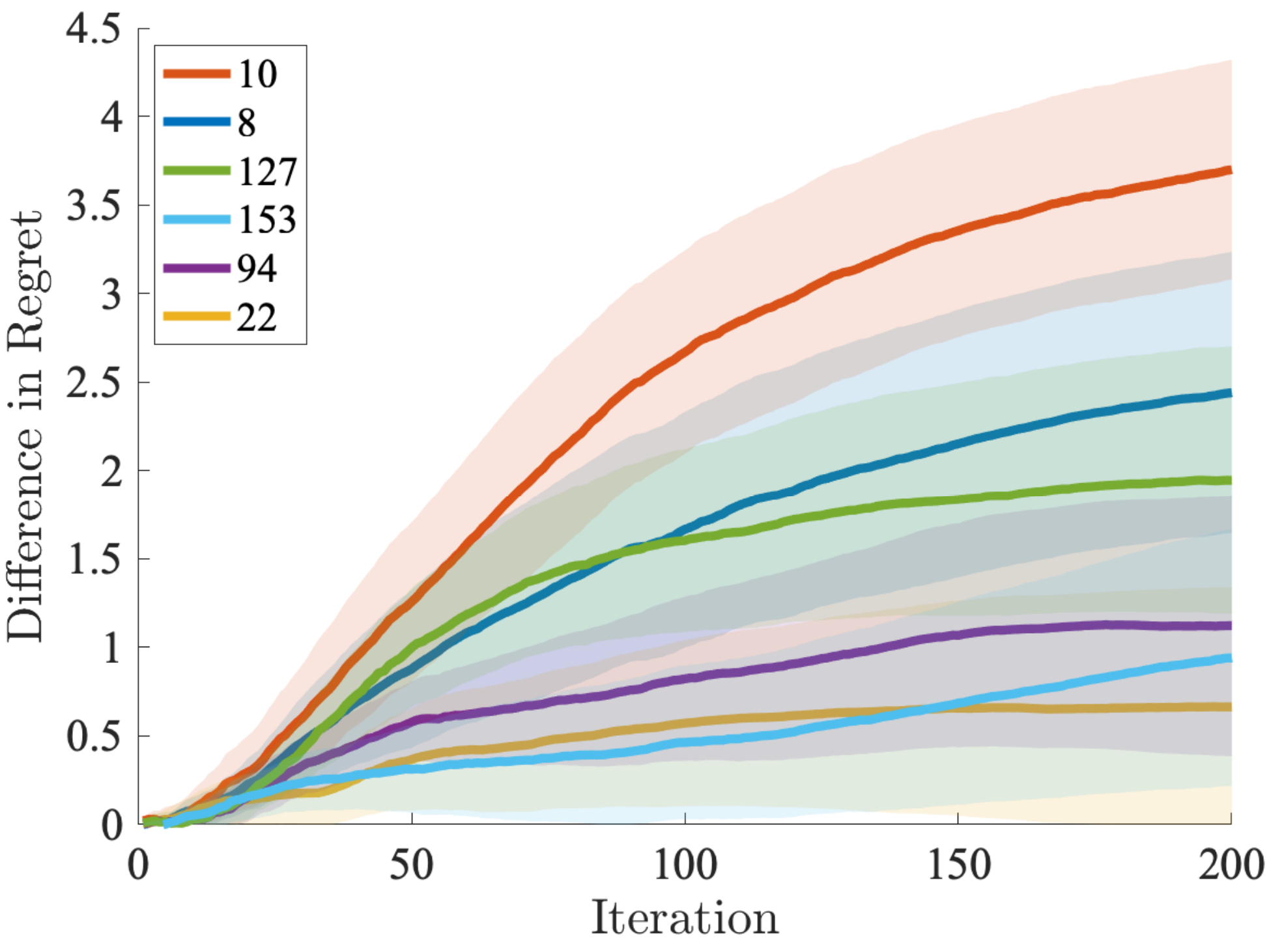}}%
	\end{minipage}
	\centering
	\begin{minipage}{0.49\columnwidth}
		\centering
		\subfigure[Difference in planning time]{\includegraphics[width=1.0\textwidth]{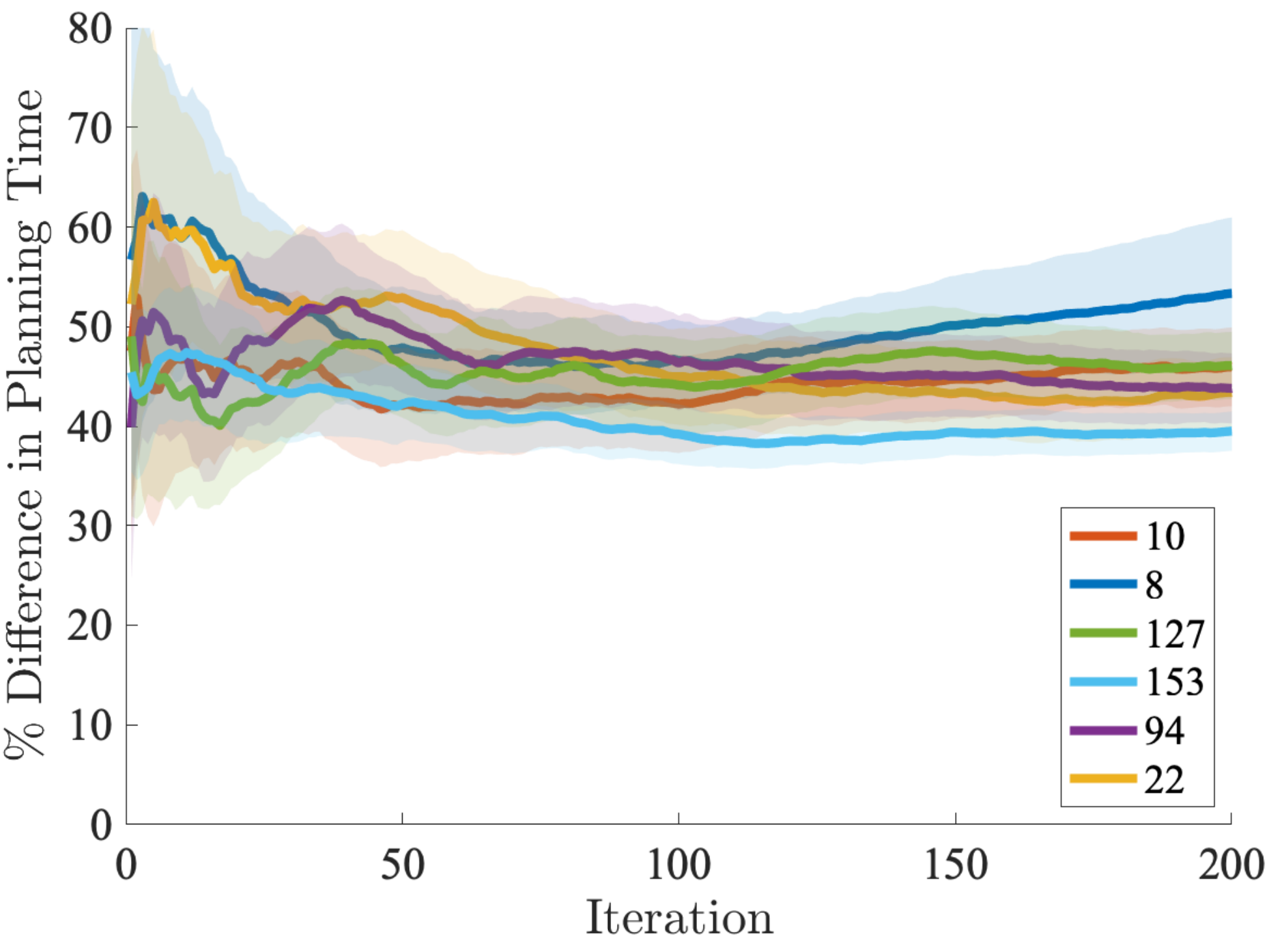}}%
	\end{minipage}
	\ifarxiv
	\else
	\fi
	\caption[]{\label{fig:approx} 95\%-confidence intervals for the (a) difference in regret and (b) percent difference in cumulative planning time when using sampling approximations vs. online mirror descent. Online mirror descent yields lower regret at the expense of longer planning times.}
	\ifarxiv
	\else
	\vspace{-10pt}
	\fi
\end{figure}

\subsection{Out-of-distribution opponents}
Now we measure performance against two agents---\textsc{OOD1} and \textsc{OOD2}---that are not in the distribution developed by our offline population synthesis approach (see Appendix \ref{sec:oodagent} for details on each agent's policy). We perform only simulated experiments, as we are unable to perform further real-world experimentation at the time of writing due to the COVID-19 pandemic. For given robustness levels $\rho / N_w \in \{0.001, 1.0\}$ and $N_w=8$ for all experiments, we perform 180 two-lap races against each of the two human-created racing agents. Again, for fair comparison, half of the experiments have the opponent start on the outside and half on the inside. Tables \ref{table:against_agent1} and \ref{table:against_agent2} show the results. Overall, the trends match those of the in-distribution opponents. Namely, adaptivity significantly increases the win-rate when robustness is high ($\rho/N_w=1.0$), whereas for low robustness ($\rho/N_w=0.001$) there is no significant change. Interestingly, adaptivity with robustness not only recovers but surpasses the win-rate of the non-adaptive non-robust policy. We hypothesize that, because out-of-distribution opponents do not match any of the learned prototypes, maintaining an uncertainty over belief automatically helps the agent plan against the ``surprising'' out-of-distribution actions. Validation of this hypothesis by comparing performance against more out-of-distribution opponents is an interesting direction for future work. Overall, we observe that even against out-of-distribution opponents, we achieve the overall goal of balancing performance and safety.

\begin{table}[t]
	\caption{The effect of adaptivity on win-rate vs. \textsc{OOD1}}
	\label{table:against_agent1}
	\begin{center}
		\begin{small}
			\begin{tabular}{lccc}
				\toprule
				& Win-rate & Win-rate & \\
				Agent     & Non-adaptive & Adaptive & p-value\\
				\midrule
				$\rho/N_w=0.001$ & 0.633$\pm$0.036 & 0.683$\pm$0.035 & 0.280\\
				$\rho/N_w=1.0$   & 0.483$\pm$0.037 & 0.717$\pm$0.034 & 5.721E-6\\
				\bottomrule
			\end{tabular}
		\end{small}
	\end{center}
	\vskip -0.2in
\end{table}
\begin{table}[t]
	\caption{The effect of adaptivity on win-rate vs. \textsc{OOD2}}
	\label{table:against_agent2}
	\begin{center}
		\begin{small}
			\begin{tabular}{lccc}
				\toprule
				& Win-rate & Win-rate & \\
				Agent     & Non-adaptive & Adaptive & p-value\\
				\midrule
				$\rho/N_w=0.001$ & 0.494$\pm$0.037 & 0.589$\pm$0.037 & 0.059\\
				$\rho/N_w=1.0$   & 0.572$\pm$0.037 & 0.739$\pm$0.033 & 0.001\\
				\bottomrule
			\end{tabular}
		\end{small}
	\end{center}
\end{table}
\section{Conclusion}
\label{sec:conclusion}
The central hypothesis of this paper is that distributionally robust evaluation of plans relative to the agent's belief state about opponents, which is updated as new observations are made, can lead to policies achieving the same performance as non-robust approaches without sacrificing safety.
To evaluate this hypothesis we identify a natural division of the underlying problem. First, we parameterize the set of possible opponents via population-based synthesis without requiring expert demonstrations. Second, we propose an online opponent-modeling framework which enables the application of distributionally robust optimization (DRO) techniques under computational constraints. 
We provide strong empirical evidence that distributional robustness combined with adaptivity enables a principled method automatically trading between safety and performance. Also, we demonstrate the transfer of our methods from simulation to real autonomous racecars. The addition of recursive feasibility arguments for stronger safety guarantees could improve the applicability of these techniques to real-world settings. Furthermore, although autonomous racing is the current focus of our experiments, future work should explore the generality of our approach in other settings such as human-robot interaction. 

\setlength{\bibsep}{3pt}
\bibliographystyle{abbrvnat}

\newpage

\appendix

\section{Offline population synthesis}
\label{sec:ops}
Here we provide extra details for Section \ref{sec:pbt}.

\paragraph*{Horizontal steps}
Horizontal steps occur as follows. Two random particles are sampled uniformly at random from adjacent temeprature levels. This forms a proposal for the swap, which is then accepted via standard MH acceptance conditions. Because the rest of the particles remain as-is, the acceptance condition reduces to a particualrly simple form (\cf~Algorithm~\ref{alg:ha}).

\ifarxiv
\begin{algorithm}[h]
	\caption{\textsc{Horizontal swap}}
	\label{alg:ha}
	\begin{algorithmic}
		\State Sample $i \sim \text{Uniform}(1,2,\ldots,L-1)$.
		\State Sample $j,k \simiid \text{Uniform}(1,2,\ldots,D)$.
		\State Sample $p \sim \text{Uniform}([0,1])$
		\State Let $a = \min \left (1, e^{ f(x^{i,j}, \theta^{i,j}) - f(x^{i+1,k}, \theta^{i+1,k})  } \right )$
		\State {\bfseries if} $p < a^{\beta_{i} - \beta_{i+1}}$
		\State $\;\;\;$ swap configurations $(x^{i,j}, \theta^{i,j})$ and $(x^{i+1,k}, \theta^{i+1,k})$
	\end{algorithmic}
\end{algorithm}
\else
\begin{algorithm}[h]
	\caption{\textsc{Horizontal swap}}
	\label{alg:ha}
	\begin{algorithmic}
		\STATE Sample $i \sim \text{Uniform}(1,2,\ldots,L-1)$.
		\STATE Sample $j,k \simiid \text{Uniform}(1,2,\ldots,D)$.
		\STATE Sample $p \sim \text{Uniform}([0,1])$
		\STATE Let $a = \min \left (1, e^{ f(x^{i,j}, \theta^{i,j}) - f(x^{i+1,k}, \theta^{i+1,k})  } \right )$
		\STATE {\bfseries if} $p < a^{\beta_{i} - \beta_{i+1}}$
		\STATE $\;\;\;$ swap configurations $(x^{i,j}, \theta^{i,j})$ and $(x^{i+1,k}, \theta^{i+1,k})$
	\end{algorithmic}
\end{algorithm}
\fi

We ran our experiments on a server with 88 Intel Xeon cores @ 2.20 GHz. Each run of 100 iterations for a given hyperparameter setting $\alpha$ took 20 hours. %
\section{Online robust planning}
\label{sec:online_appendix}

Here we provide extra details for Section \ref{sec:bandits}.

\subsection{Solving problem \eqref{eq:real-problem}}
We can rewrite the constraint $\fdiv{q}{\mathbf{1}/N_w} \le \tol$ as $\|q-\mathbf{1}/N_w\|^2 \le \rho/N_w$. Then, the partial Lagrangian can be written as
\begin{equation*}
  \mc{L}(q, \lambda) = \sum_i q_i c_i(t)
  - \frac{\lambda}{2}\left (  \|q-\mathbf{1}/N_w\|^2 - \tol/N_w \right ).
\end{equation*}
By inspection of the right-hand side, we see that, for a given $\lambda$, finding $v(\lambda)=\sup_{q \in \Delta}\mc{L}({q, \lambda})$ is equivalent to a Euclidean-norm projection of the vector $\mathbf{1}/N_w + c(t)/\lambda$ onto the probability simplex $\Delta$. This latter problem is directly amenable to the methods of \citet{Duchi_2008}.

\subsection{Proof of Proposition \ref{prop:approx}}
We redefine notation to suppress dependence of the cost $C$ on other variables and just make explicit the dependence on the random index $J$. Namely, we let $C: \mc{J} \to [-1, 1]$ be a function of the random index $J$. We consider the convergence of
\begin{equation*}
  \sup_{Q \in \mc{P}_{N_w}}\E_Q[C(J)] ~~ \mbox{to} ~~ \sup_{Q \in \mc{P}}
  \E_Q[C(J)].
\end{equation*}
To ease notation, we hide dependence on $J$ and
for a sample $J_1, \ldots, J_{N_w}$ of random vectors $J_k$, we denote $C_k:= C(J_k)$ for shorthand, so that the $C_k$ are bounded independent random variables. Our proof technique is similar in style to that of \citet{sinha2016learning}. We provide proofs for technical lemmas that follow in support of Proposition \ref{prop:approx} that are shorter and more suitable for our setting (in particular Lemmas \ref{lemma:E-sup-lower-bound} and \ref{lemma:sqrt-moments}).

Treating $C = (C_1, \ldots, C_{N_w})$ as a vector, the mapping
$C \mapsto \sup_{Q \in \mc{P}_{N_w}} \E_Q[C]$ is a $\sqrt{\tol+1} / \sqrt{N_w}$-Lipschitz convex function
of independent bounded random variables. Indeed, letting $q \in \R^{N_w}_+$ be
the empirical probability mass function associated with $Q \in \mc{P}_{N_w}$, we have
$\frac{1}{N_w} \sum_{i = 1}^{N_w} (N_w q_i)^2 \le \tol + 1$ or $\ltwo{q} \le \sqrt{(1+\tol)/N_w}$.
Using Samson's sub-Gaussian concentration
inequality~\cite{Samson00}
for Lipschitz convex functions of bounded random
variables, we have with probability at least $1-\delta$ that
\begin{equation}
  \sup_{Q \in \mc{P}_{N_w}}
  \E_Q[C]
  \in \E\left[\sup_{Q \in \mc{P}_{N_w}} \E_Q[C]\right]
  \pm 2\sqrt{2} \sqrt{\frac{(1 + \tol)\log \frac{2}{\delta}}{N_w}}.
  \label{eqn:high-prob-to-expect}
\end{equation}
By the containment~\eqref{eqn:high-prob-to-expect}, we only need to consider convergence of 
\begin{equation*}
  \E\left[\sup_{Q \in  \mc{P}_{N_w}} \E_Q[C]\right]
  ~~ \mbox{to} ~~
  \sup_{Q \in \mc{P}} \E_Q[C],
\end{equation*}
which we do with the following lemma.
\begin{lemma}[\citet{sinha2016learning}]\label{lemma:E-sup-lower-bound}  
Let $Z = (Z_1, \ldots, Z_{N_w})$ be a random vector of independent random
  variables $Z_i \simiid P_0$, where $|Z_i| \le M$ with probability $1$. Let
  $C_\tol = \frac{2(\tol + 1)}{\sqrt{1 + \tol} - 1}$. Then
  \begin{align*}
    \E\left[\sup_{Q \in \mc{P}_{N_w} }\E_Q[Z]\right]
    \ge \sup_{Q \in \mc{P}} \E_Q[Z] - 4C_{\tol}M\sqrt \frac{\log (2N_w)}{N_w}
  \end{align*}
  and
  \begin{equation*}
    \E\left[\sup_{Q \in \mc{P}_{N_w}} \E_Q[Z]\right]
    \le \sup_{Q \in \mc{P}} \E_Q[Z].
  \end{equation*}
\end{lemma}
\noindent See Appendix \ref{sec:first-tech-lemma} for the proof.

Combining Lemma~\ref{lemma:E-sup-lower-bound} with
containment~\eqref{eqn:high-prob-to-expect} gives the result.

\subsection{Proof of Lemma \ref{lemma:E-sup-lower-bound}}\label{sec:first-tech-lemma}
Before beginning the proof, we first state a technical lemma.
\begin{lemma}[\citet{Ben-TalHeWaMeRe13}]
  \label{lemma:duality}
  Let $\phi$ be any closed convex function with domain $\dom \phi \subset
  \openright{0}{\infty}$, and let
  $\phi^*(s) = \sup_{t \ge 0} \{ts - \phi(t)\}$ be its conjugate. Then
  for any distribution $P$ and any function $g : \mc{W} \to \R$ we have
\ifarxiv
  \begin{equation*}
    \sup_{Q : \fdiv{Q}{P} \le \tol}
    \int g(w) dQ(w) = \inf_{\lambda \ge 0, \eta}
    \left\{\lambda \int \phi^*\left(\frac{g(w) - \eta}{\lambda}\right) dP(w)
    + \tol \lambda + \eta \right\}.
  \end{equation*}
\else
  \begin{multline*}
    \sup_{Q : \fdiv{Q}{P} \le \tol}
    \int g(w) dQ(w) \\
    = \inf_{\lambda \ge 0, \eta}
    \left\{\lambda \int \phi^*\left(\frac{g(w) - \eta}{\lambda}\right) dP(w)
    + \tol \lambda + \eta \right\}.
  \end{multline*}
\fi
\end{lemma}
\noindent See Appendix \ref{sec:tech-lemma-pf} for the proof.

We prove the result for general $\phi$-divergences $\phi(t)=t^k-1$, $k\ge 2$. To simplify algebra, we work with a scaled version of the $\phi$-divergence:
$\phi(t) = \tfrac{1}{k}(t^k-1)$, so the scaled population and empirical constraint sets we consider are defined by
\ifarxiv
\else
\begin{footnotesize}
\fi
\begin{equation*}
  \mc{P} = \left\{Q : \fdiv{Q}{P_0} \le \frac{\tol}{k} \right\}
  ~~ \mbox{and} ~~
  \mc{P}_{N_w}:=\left \{q: \fdiv{q}{\mathbf{1}/{N_w}} \le \frac{\tol}{k}\right \}.
\end{equation*}
\ifarxiv
\else
\end{footnotesize}
\fi
Then by Lemma~\ref{lemma:duality}, we obtain
\ifarxiv
\else
\begin{scriptsize}
\fi
\begin{align*}
   \E\left[\sup_{Q \in \mc{P}_{N_w}} \E_Q[Z]\right] &= \E_{P_0}\left[\inf_{\lambda \ge 0, \eta} \frac{1}{N_w} \sum_{i = 1}^{N_w} \lambda \phi^*\left(\frac{Z_i - \eta}{\lambda} \right) + \eta + \frac{\tol}{k} \lambda\right]\\
  &\le \inf_{\lambda \ge 0, \eta} \E_{P_0}\left[\frac{1}{N_w} \sum_{i = 1}^{N_w} \lambda \phi^*\left(\frac{Z_i - \eta}{\lambda}\right)+ \eta + \frac{\tol}{k} \lambda \right] \\
  &= \inf_{\lambda \ge 0, \eta}\left\{\E_{P_0}\left[\lambda \phi^*\left(\frac{Z - \eta}{\lambda}\right)\right]+ \frac{\tol}{k} \lambda + \eta\right\}\\
  &= \sup_{Q \in \mc{P}} \E_Q[Z].
\end{align*}
\ifarxiv
\else
\end{scriptsize}\fi This proves the upper bound in Lemma \ref{lemma:E-sup-lower-bound}.

Now we focus on the lower bound. For the function $\phi(t) = \frac{1}{k} (t^k  - 1)$, we have
$\phi^*(s) = \frac{1}{k^*} \hinge{s}^{k^*} + \frac{1}{k}$, where $1/k^* + 1/k=1$, so that
the duality result in Lemma~\ref{lemma:duality} gives
\begin{small}
\begin{equation*}
  \sup_{Q \in \mc{P}_{N_w}}
  \E_Q[Z]
  = \inf_\eta \left\{ \left ({1 + \tol}\right)^{1/k}
  \bigg(\frac{1}{N_w} \sum_{i=1}^{N_w} \hinge{Z_i - \eta}^{k^*}\bigg)^\frac{1}{k^*}
  + \eta \right\}.
\end{equation*}
\end{small}Because $|Z_i| \le M$ for all $i$, we claim that any $\eta$ minimizing the
preceding expression must satisfy
\begin{equation}
  \label{eqn:eta-range}
  \eta \in \left[-\frac{1 + (1 + \tol)^{\frac{1}{k^*}}}{({1 + \tol})^{\frac{1}{k^*}} - 1},
    1\right] \cdot M.
\end{equation}
For convenience, we first define the shorthand
\begin{equation*}
  S_{N_w}(\eta)
  \defeq  \left ({1 + \tol}\right)^{1/k}
  \bigg(\frac{1}{N_w} \sum_{i=1}^{N_w} \hinge{Z_i - \eta}^{k^*}\bigg)^\frac{1}{k^*}
  + \eta.
\end{equation*}
Then it is clear that $\eta \le M$, because otherwise we would have
$S_{N_w}(\eta) > M \ge \inf_\eta S_{N_w}(\eta)$.
Let the lower bound be of the form $\eta = -c M$ for some $c > 1$. Taking derivatives
of the objective $S_{N_w}(\eta)$ with respect to $\eta$, we have
\begin{align*}
  S'_{N_w}(\eta)
  &= 1 - ({1 + \tol})^{1/k}
  \frac{\frac{1}{N_w} \sum_{i = 1}^{N_w} \hinge{Z_i - \eta}^{k^*-1}}{
    \left ({\frac{1}{N_w} \sum_{i=1}^{N_w} \hinge{Z_i - \eta}^{k^*}}\right)^{1-\frac{1}{k^*}}}\\
  &\le 1 - ({1 + \tol})^{1/k} \left ( \frac{(c - 1) M}{(c + 1) M}\right ) ^{k^*-1}\\
  &= 1 - ({1 + \tol})^{1/k} \left (\frac{c - 1}{c + 1}\right )^{k^*-1}.
\end{align*}
For any $c > c_{\tol, k}\defeq \frac{(1 + \tol)^\frac{1}{k^*} + 1}{(1 +
  \tol)^\frac{1}{k^*} - 1}$, the preceding display is negative, so
we must have $\eta \ge -c_{\rho,k} M$. For the remainder of the proof, we thus define the interval
\begin{equation*}
  U \defeq \left[-M c_{\tol, k}, M\right],
  ~~
  c_{\tol,k} = \frac{(1 + \tol)^\frac{1}{k^*} + 1}{(1 + \tol)^\frac{1}{k^*} - 1},
\end{equation*}
and we assume w.l.o.g.\ that $\eta \in U$.

Again applying the duality result of Lemma~\ref{lemma:duality},
we have that
\ifarxiv
\begin{align}\label{eqn:inf-sup-snw}
\nonumber\E\left[\sup_{Q \in \mc{P}_{N_w}}\E_Q[Z]\right] &= \E\left[\inf_{\eta \in U} S_{N_w}(\eta\right]\\
\nonumber &= \E\left[\inf_{\eta \in U}\{ S_{N_w}(\eta) - \E[S_{N_w}(\eta)] + \E[S_{N_w}(\eta)]\}\right] \\
& \ge \inf_{\eta \in U} \E[S_{N_w}(\eta)]-\E\left[\sup_{\eta \in U} |S_{N_w}(\eta) - \E[S_{N_w}(\eta)]|
    \right].
\end{align}
\else
\begin{small}
\begin{align}\label{eqn:inf-sup-snw}
\nonumber\E\left[\sup_{Q \in \mc{P}_{N_w}}\E_Q[Z]\right] &= \E\left[\inf_{\eta \in U} S_{N_w}(\eta\right]\\
\nonumber &= \E\left[\inf_{\eta \in U}\{ S_{N_w}(\eta) - \E[S_{N_w}(\eta)] + \E[S_{N_w}(\eta)]\}\right] \\
\nonumber &\ge \inf_{\eta \in U} \E[S_{N_w}(\eta)]\\
   & -\E\left[\sup_{\eta \in U} |S_{N_w}(\eta) - \E[S_{N_w}(\eta)]|
    \right].
\end{align}
\end{small}\fi To bound the first term in expression~\eqref{eqn:inf-sup-snw}, we use the following lemma.
\begin{lemma}[\citet{sinha2016learning}]
  \label{lemma:sqrt-moments}
 Let $Z\ge 0, Z \not \equiv 0$ be a random variable with finite $2p$-th
  moment for $1\le p \le 2$. Then we have the following inequality:
  \begin{subequations}
    \begin{equation}
        \E\left[\bigg(\frac{1}{n} \sum_{i=1}^n Z_i^p\bigg)^\frac{1}{p} \right] \ge \|Z\|_p - \frac{p-1}{p}\sqrt\frac{2}{n} \sqrt{\var(Z^p / \E[Z^p])}\|Z\|_2,
    \end{equation}
    and if $\linf{Z} \le C$, then
    \begin{equation}
\E\left[\bigg(\frac{1}{n} \sum_{i=1}^n Z_i^p\bigg)^\frac{1}{p} \right] \ge \|Z\|_p - C\frac{p-1}{p}\sqrt\frac{2}{n}.
\end{equation}
  \end{subequations}
\end{lemma}
\noindent See Appendix \ref{sec:last-tech-lemma-pf} for the proof. Now, note that
$\hinge{Z - \eta} \in [0, 1+c_{\tol, k}] M$ and
$({1 + \tol})^{1/k} (1+c_{\tol, k}) =: C_{\tol, k}$. Thus, by Lemma
\ref{lemma:sqrt-moments} we obtain that
\ifarxiv
\begin{align*}
  \E[S_{N_w}(\eta)]
  &\ge ({1 + \tol})^{1/k} \E\left[\hinge{Z - \eta}^{k^*}\right]^{1/k^*} + \eta
  - C_{\rho,k} M \frac{k^*-1}{k^*}\sqrt\frac{2}{N_w}.
\end{align*}
\else
\begin{align*}
  \E[S_{N_w}(\eta)]
  &\ge ({1 + \tol})^{1/k} \E\left[\hinge{Z - \eta}^{k^*}\right]^{1/k^*}\\
  &+ \eta
  - C_{\rho,k} M \frac{k^*-1}{k^*}\sqrt\frac{2}{N_w}.
\end{align*}
\fi
Using that $\frac{k^* - 1}{k^*} = \frac{1}{k}$, taking the infimum over
$\eta$ on the right hand side and using duality yields
\begin{equation*}
  \inf_\eta \E[S_{N_w}(\eta)]
  \ge \sup_{Q \in \mc P}
  \E_Q[Z]
  - C_{\rho,k} \frac{M}{k} \sqrt\frac{2}{N_w}.
\end{equation*}

To bound the second term in expression~\eqref{eqn:inf-sup-snw}, we use concentration results
for Lipschitz functions. First, the function $\eta \mapsto S_{N_w}(\eta)$ is $\sqrt{1 + \tol}$-Lipschitz 
in $\eta$. To see this, note that for $1\le k^\star\le 2$ and $X\ge 0$, by Jensen's inequality,
  \begin{equation*}
  \frac{\E [X^{k^\star-1}]}{(\E[X^{k^\star}])^{1-1/k^\star}} \le \frac{\E [X]^{k^\star-1}}{(\E[X^{k^\star}])^{1-1/k^\star}} \le \frac{\E [X]^{k^\star-1}}{\E[X]^{k^\star-1}}=1,
  \end{equation*}
  so $S_{N_w}'(\eta) \in [1-(1+\tol)^{\frac{1}{k}},1]$ and therefore $S_{N_w}$ is $(1 + \tol)^{1/k}$-Lipschitz in $\eta$.
Furthermore, the mapping $T: z \mapsto (1 + \tol)^\frac{1}{k} (\frac{1}{N_w}
\sum_{i = 1}^{N_w} \hinge{z_i - \eta}^{k^*})^\frac{1}{k^*}$ for $z \in \R^{N_w}$ is convex and $(1 + \tol)^\frac{1}{k} / \sqrt{N_w}$-Lipschitz. This is verified by the following:
\ifarxiv
\else
\begin{scriptsize}
\fi
\begin{align*}
\left \lvert T(z) - T(z') \right \rvert & \le \left ({1 + \tol}\right)^{1/k} \left \lvert  \bigg( \frac{1}{N_w}\sum_{i=1}^{N_w} \left |\hinge{z_i - \eta} - \hinge{z'_i - \eta}\right |^{k^*}\bigg)^\frac{1}{k^*} \right \rvert \\
& \le \frac{\left ({1 + \tol}\right)^{1/k}}{{N_w}^{1/{k^*}}}\left \lvert \bigg( \sum_{i=1}^{N_w} \left \lvert z_i - z'_i \right \rvert^{k^*}\bigg)^\frac{1}{k^*} \right \rvert \\
& \le \frac{\left ({1 + \tol}\right)^{1/k}}{\sqrt{N_w}}\|z-z'\|_2, 
\end{align*}
\ifarxiv
\else
\end{scriptsize}
\fi
where the first inequality is Minkowski's inequality and the third inequality follows from the fact that for
any vector $x \in \R^n$, we have $\norm{x}_p \le n^\frac{2 - p}{2p} \ltwo{x}$
for $p \in [1, 2]$, where these denote the usual vector norms. Thus, the mapping
$Z \mapsto S_{N_w}(\eta)$ is $(1 + \tol)^{1/k} / \sqrt{N_w}$-Lipschitz continuous
with respect to the $\ell_2$-norm on $Z$.
Using Samson's sub-Gaussian concentration result for
convex Lipschitz functions, we have
\begin{equation*}
  \P\left(|S_{N_w}(\eta) - \E[S_{N_w}(\eta)]|
    \ge \delta \right)
  \le 2 \exp\left(-\frac{N_w \delta^2}{2 C_{\tol, k}^2 M^2}\right)
\end{equation*}
for any fixed $\eta \in \R$ and any $\delta \ge 0$.  Now, let $\mc{N}(U,
\epsilon) = \{\eta_1, \ldots, \eta_{N(U, \epsilon)}\}$ be an $\epsilon$
cover of the set $U$, which we may take to have size at most $N(U,
\epsilon) \le  M(1+c_{\tol, k}) \frac{1}{\epsilon}$. Then we have
\ifarxiv
\begin{align*}
  \sup_{\eta \in U} |S_{N_w}(\eta) - \E[S_{N_w}(\eta)] \le \max_{i \in \mc{N}(U, \epsilon)}
  |S_{N_w}(\eta_i) - \E[S_{N_w}(\eta_i)]|
  + \epsilon (1 + \tol)^{1/k}.
\end{align*}
\else
\begin{align*}
  \sup_{\eta \in U} |S_{N_w}(\eta) - \E[S_{N_w}(\eta)]\\
  \le \max_{i \in \mc{N}(U, \epsilon)}
  |S_{N_w}(\eta_i) - \E[S_{N_w}(\eta_i)]|
  + \epsilon (1 + \tol)^{1/k}.
\end{align*}\fi
Using the fact that 
$\E[\max_{i \le n} |X_i|] \le \sqrt{2 \sigma^2 \log(2n)}$ for
$X_i$ all $\sigma^2$-sub-Gaussian,
we have
\ifarxiv
\begin{align*}
  \E\left[\max_{i \in \mc{N}(U, \epsilon)}
    |S_{N_w}(\eta_i) - \E[S_{N_w}(\eta_i)]|\right] \le C_{\tol, k}\sqrt{2\frac{M^2}{N_w}
    \log 2N(U, \epsilon)}.
\end{align*}
\else
\begin{align*}
  \E\left[\max_{i \in \mc{N}(U, \epsilon)}
    |S_{N_w}(\eta_i) - \E[S_{N_w}(\eta_i)]|\right]&\\
  \le C_{\tol, k}\sqrt{2\frac{M^2}{N_w}
    \log 2N(U, \epsilon)}.
\end{align*}\fi
Taking $\epsilon = M(1+c_{\tol, k}) / {N_w}$ gives that
\ifarxiv
\begin{align*}
  \E\left[\sup_{\eta \in U} |S_{N_w}(\eta) - \E[S_{N_w}(\eta)]
  \right] \le \sqrt{2}MC_{\tol, k}  \sqrt{\frac{1}{N_w} \log (2N_w)} + \frac{C_{\tol, k}M}{N_w}.
\end{align*}
\else
\begin{align*}
  \E\left[\sup_{\eta \in U} |S_{N_w}(\eta) - \E[S_{N_w}(\eta)]
  \right]\\
  \le \sqrt{2}MC_{\tol, k}  \sqrt{\frac{1}{N_w} \log (2N_w)} + \frac{C_{\tol, k}M}{N_w}.\\
\end{align*}\fi
Then, in total we have (using $C_\tol \ge C_{\tol,k}$, $k\ge 2$, and $N_w \ge 1$),
\ifarxiv
  \begin{align*}
    \E\left[\sup_{Q \in \mc{P}_{N_w} }\E_Q[Z]\right]
    &\ge \sup_{Q \in \mc{P}} \E_Q[Z] - \frac{C_{\tol} M \sqrt{2}}{\sqrt{N_w}}
    \left( \frac{1}{k}
    +  \sqrt{\log (2N_w)}
    + \frac{1}{\sqrt{2N_w}}\right)\\
    & \ge \sup_{Q \in \mc{P}} \E_Q[Z] - 4C_{\tol}M\sqrt \frac{\log (2N_w)}{N_w}.
  \end{align*}
\else
  \begin{align*}
    \E\left[\sup_{Q \in \mc{P}_{N_w} }\E_Q[Z]\right]
    &\ge \sup_{Q \in \mc{P}} \E_Q[Z] - \\
    &\frac{C_{\tol} M \sqrt{2}}{\sqrt{N_w}}
    \left( \frac{1}{k}
    +  \sqrt{\log (2N_w)}
    + \frac{1}{\sqrt{2N_w}}\right)\\
    & \ge \sup_{Q \in \mc{P}} \E_Q[Z] - 4C_{\tol}M\sqrt \frac{\log (2N_w)}{N_w}.
  \end{align*} \fi
This gives the desired result of the lemma.

\subsection{Proof of Lemma \ref{lemma:duality}}\label{sec:tech-lemma-pf}
Let $L \ge 0$ satisfy $L(w) = dQ(w) / dP(w)$, so that $L$ is the
likelihood ratio between $Q$ and $P$. Then we have
\ifarxiv
\begin{align*}
  &\lefteqn{\sup_{Q : \fdiv{Q}{P} \le \tol}
    \int g(w) dQ(w)
    = \sup_{\int \phi(L) dP \le \tol,
      \E_P[L] = 1} \int g(w) L(w) dP(w)} \\
 & = \sup_{L \ge 0} \inf_{\lambda \ge 0, \eta}
  \bigg\{\int g(w) L(w) dP(w) - \lambda \left(\int f(L(w)) dP(w) - \tol\right) - \eta \left(\int L(w) dP(w) - 1 \right) \bigg\} \\
 & = \inf_{\lambda \ge 0, \eta}
  \sup_{L \ge 0} \bigg\{\int g(w) L(w) dP(w) - \lambda \left(\int f(L(w))
  dP(w) - \tol\right) - \eta \left(\int L(w) dP(w) - 1 \right) \bigg\},
\end{align*}
\else
\begin{tiny}
\begin{align*}
  &\lefteqn{\sup_{Q : \fdiv{Q}{P} \le \tol}
    \int g(w) dQ(w)
    = \sup_{\int \phi(L) dP \le \tol,
      \E_P[L] = 1} \int g(w) L(w) dP(w)} \\
 & = \sup_{L \ge 0} \inf_{\lambda \ge 0, \eta}
  \bigg\{\int g(w) L(w) dP(w) - \lambda \left(\int f(L(w)) dP(w) - \tol\right)\\
 & - \eta \left(\int L(w) dP(w) - 1 \right) \bigg\} \\
 & = \inf_{\lambda \ge 0, \eta}
  \sup_{L \ge 0} \bigg\{\int g(w) L(w) dP(w) - \lambda \left(\int f(L(w))
  dP(w) - \tol\right)\\
 & - \eta \left(\int L(w) dP(w) - 1 \right) \bigg\},
\end{align*}
\end{tiny}\fi where we have used that strong duality obtains because the problem is
strictly feasible in its non-linear constraints (take $L \equiv 1$), so that
the extended Slater condition holds~\citep[Theorem 8.6.1 and Problem 8.7]{Luenberger69}.
Noting that $L$ is simply a positive (but otherwise arbitrary) function,
we obtain
\begin{align*}
  &\sup_{Q : \fdiv{Q}{P} \le \tol}
    \int g(w) dQ(w)\\
  &= \inf_{\lambda \ge 0, \eta}
  \int \sup_{\ell \ge 0} \left\{(g(w) - \eta) \ell - \lambda \phi(\ell)\right\}
  dP(w) + \lambda \tol + \eta \\
  &= \inf_{\lambda \ge 0, \eta}
  \int \lambda \phi^*\left(\frac{g(w) - \eta}{\lambda} \right) dP(w)
  + \eta + \tol \lambda.
\end{align*}
Here we have used that $\phi^*(s) = \sup_{t \ge 0}\{st - \phi(t)\}$ is the
conjugate of $\phi$ and that $\lambda \ge 0$, so that we may take divide
and multiply by $\lambda$ in the supremum calculation.

\subsection{Proof of Lemma \ref{lemma:sqrt-moments}}\label{sec:last-tech-lemma-pf}
For $a > 0$, we have
\begin{equation*}
  \inf_{\lambda \ge 0}
  \left\{\frac{a^p}{p \lambda^{p-1}} + \lambda\frac{p-1}{p}
  \right\} = a,
\end{equation*}
(with $\lambda = a$ attaining the infimum), and taking derivatives yields
\begin{equation*}
  \frac{a^p}{p \lambda^{p-1}} + \lambda\frac{p-1}{p}
  \ge \frac{a^p}{p \lambda_1^{p-1}} + \lambda_1\frac{p-1}{p}
  + \frac{p-1}{p} \left(1 - \frac{a^p}{\lambda_1^p} \right)
  (\lambda - \lambda_1).
\end{equation*}

Using this in the moment expectation, by setting
$\lambda_n = \sqrt[p]{\frac{1}{n} \sum_{i=1}^n Z_i^p}$, we have
for any $\lambda \ge 0$ that
\ifarxiv
\begin{align*}
  \E\left[\bigg(\frac{1}{n} \sum_{i=1}^n Z_i^p\bigg)^\frac{1}{p} \right]
  & = \E\left[\frac{\sum_{i=1}^n Z_i^p}{p n \lambda_n^{p-1}}
    + \lambda_n\frac{p-1}{p} \right] \\
  & \ge \E\left[\frac{\sum_{i=1}^n Z_i^p}{p n \lambda^{p-1}}
    + \lambda\frac{p-1}{p} \right] + \frac{p-1}{p}\E\left[\left(1 -\frac{\sum_{i=1}^n Z_i^p}{n \lambda^p}\right)
    (\lambda_n - \lambda)\right].
\end{align*}
\else
\begin{align*}
  \E\left[\bigg(\frac{1}{n} \sum_{i=1}^n Z_i^p\bigg)^\frac{1}{p} \right]
  & = \E\left[\frac{\sum_{i=1}^n Z_i^p}{p n \lambda_n^{p-1}}
    + \lambda_n\frac{p-1}{p} \right] \\
  & \ge \E\left[\frac{\sum_{i=1}^n Z_i^p}{p n \lambda^{p-1}}
    + \lambda\frac{p-1}{p} \right]\\
  &+ \frac{p-1}{p}\E\left[\left(1 -\frac{\sum_{i=1}^n Z_i^p}{n \lambda^p}\right)
    (\lambda_n - \lambda)\right].
\end{align*}\fi

Now we take $\lambda = \|Z\|_p$, and we apply the Cauchy-Schwarz
inequality to obtain
\ifarxiv
\begin{align*}
  \E\left[\bigg(\frac{1}{n} \sum_{i=1}^n Z_i^p\bigg)^\frac{1}{p} \right]
  & \ge \|Z\|_p - \frac{p-1}{p}\E\left[\left(1 - \frac{\frac{1}{n} \sum_{i=1}^n Z_i^p}{
     \|Z\|_p^p}\right)^2\right]^\half
  \E\left[\left(\bigg(\frac{1}{n} \sum_{i=1}^n Z_i^p\bigg)^\frac{1}{p}
    - \|Z\|_p \right)^2\right]^\half \nonumber \\
   &= \|Z\|_p - \frac{p-1}{p \sqrt{n}} \sqrt{\var(Z^p / \E[Z^p])}
  \E\left[\left(\bigg(\frac{1}{n} \sum_{i=1}^n Z_i^p\bigg)^\frac{1}{p}
    - \E[Z^p]^\frac{1}{p} \right)^2\right]^\half
  \label{eqn:zp-minus-zp} \\
   & \ge \|Z\|_p - \frac{p-1}{p \sqrt{n}} \sqrt{\var(Z^p / \E[Z^p])}
  \E\left[\bigg(\frac{1}{n} \sum_{i=1}^n Z_i^p\bigg)^\frac{2}{p}
    + \E[Z^p]^\frac{2}{p} \right]^\half \nonumber\\
    & \ge \|Z\|_p - \frac{p-1}{p}\sqrt\frac{2}{n} \sqrt{\var(Z^p / \E[Z^p])}\|Z\|_2,\nonumber
\end{align*}
\else
\begin{scriptsize}
\begin{multline*}
  \E\left[\bigg(\frac{1}{n} \sum_{i=1}^n Z_i^p\bigg)^\frac{1}{p} \right]
  \ge \|Z\|_p\\
  - \frac{p-1}{p}\E\left[\left(1 - \frac{\frac{1}{n} \sum_{i=1}^n Z_i^p}{
     \|Z\|_p^p}\right)^2\right]^\half
  \E\left[\left(\bigg(\frac{1}{n} \sum_{i=1}^n Z_i^p\bigg)^\frac{1}{p}
    - \|Z\|_p \right)^2\right]^\half \nonumber \\
   = \|Z\|_p - \frac{p-1}{p \sqrt{n}} \sqrt{\var(Z^p / \E[Z^p])}
  \E\left[\left(\bigg(\frac{1}{n} \sum_{i=1}^n Z_i^p\bigg)^\frac{1}{p}
    - \E[Z^p]^\frac{1}{p} \right)^2\right]^\half
  \label{eqn:zp-minus-zp} \\
   \ge \|Z\|_p - \frac{p-1}{p \sqrt{n}} \sqrt{\var(Z^p / \E[Z^p])}
  \E\left[\bigg(\frac{1}{n} \sum_{i=1}^n Z_i^p\bigg)^\frac{2}{p}
    + \E[Z^p]^\frac{2}{p} \right]^\half \nonumber\\
    \ge \|Z\|_p - \frac{p-1}{p}\sqrt\frac{2}{n} \sqrt{\var(Z^p / \E[Z^p])}\|Z\|_2,\nonumber
\end{multline*}
\end{scriptsize}\fi where the last inequality follows by the fact that the norm is non-decreasing in
$p$. 

In the case that we have the unifom bound $\|Z\|_{\infty} \le C$, we can get
tighter guarantees. To that end, we state a simple lemma.
\begin{lemma}
  \label{lemma:holder-thing}
  For any random variable $X \ge 0 $ and $a \in [1,2]$, we have
  \begin{equation*}
    \E[X^{ak}] \le \E[X^k]^{2-a}\E[X^{2k}]^{a-1}
  \end{equation*}
\end{lemma}
\begin{proof}
  For $c \in [0,1]$, $1/p + 1/q=1$ and $A\ge 0$, we have by Holder's inequality,
  \begin{equation*}
    \E[A] = \E[A^{c}A^{1-c}] \le \E[A^{pc}]^{1/p}\E[A^{q(1-c)}]^{1/q}
  \end{equation*}
  Now take $A:=X^{ak}$, $1/p=2-a$, $1/q = a-1$, and $c=\frac{2}{a}-1$.
\end{proof}
First, note that $\E[Z^{2p}]\le C^p\E[Z^p]$. For
$ 1\le p \le 2$, we can take $a = 2/p$ in Lemma \ref{lemma:holder-thing}, so
that we have
\begin{equation*}
E[Z^2] \le \E[Z^p]^{2-\frac{2}{p}}\E[Z^{2p}]^{\frac{2}{p}-1}\le\|Z\|_p^pC^{2-p}.
\end{equation*}
Now, we can plug these into the expression above (using $\var{Z^p}\le \E[Z^{2p}] \le C^p \|Z\|_p^p$), yielding
\begin{equation*}
\E\left[\bigg(\frac{1}{n} \sum_{i=1}^n Z_i^p\bigg)^\frac{1}{p} \right] \ge \|Z\|_p - C\frac{p-1}{p}\sqrt\frac{2}{n}
\end{equation*}
as desired.

\subsection{Proof of Proposition \ref{prop:exp3}}

We utilize the following lemma for regret of online mirror descent.
\begin{lemma}\label{thm:omd}
The expected regret for online mirror descent with unbiased stochastic subgradient $\gamma(t)$ and stepsize $\eta$ is
\small{
\begin{equation}\label{eq:general-regret}
\sum_{t=1}^T \E \left [ \gamma(t)^T(w(t) - w^{\star}) \right ] \le \frac{\log(d)}{\eta} + \frac{\eta}{2}\E \left [ \sum_{t=1}^T \sum_{j=1}^dw_j(t) \gamma_j(t)^2 \right ]
\end{equation}
}
\end{lemma}

See Appendix \ref{sec:lemmapf} for the proof. Now we bound the right-hand term of the regret bound \eqref{eq:general-regret} in our setting. For this we utilize the following:
\begin{small}
\begin{align*}
\E \left [ \gamma_i(t)^2  | w(t) \right ] &= \frac{1}{N_w^2} \frac{L_i^2(t)}{w_i^2(t)} \E \left [ \left (  \sum_{k=1}^{N_w} \mathbf{1}\left \{J_k=i \right \} \right )^2  \bigg\vert w(t) \right ] \\
&= \frac{1}{N_w^2} \frac{L_i^2(t)}{w_i^2(t)} \left (N_w (N_w-1)w_i(t)^2 + N_w w_i(t) \right ),
\end{align*}
\end{small}
where the latter fact is simply the second moment for the sum of $N_w\;$ random variables $\simiid \text{Bernoulli}(w_i(t))$. Then, 
\begin{small}
\begin{align*}
\sum_{i=1}^{d}w_i(t) \E \left [ \gamma_i(t)^2  | w(t) \right ] &= \sum_{i=1}^{d} L_i(t)^2  \left ( \frac{N_w-1}{N_w}w_i(t) + \frac{1}{N_w} \right )\\
& \le \sum_{i=1}^{d} \left ( \frac{N_w-1}{N_w}w_i(t) + \frac{1}{N_w} \right )\\
& = \frac{N_w-1}{N_w} + \frac{d}{N_w}\\
& =: z.
\end{align*}
\end{small}
Plugging in the prescribed $\eta=\sqrt{\frac{2\log(d)}{zT}}$ into the bound \eqref{eq:general-regret} yields the result.

\subsection{Proof of Lemma \ref{thm:omd}}\label{sec:lemmapf}
We first show the more general regeret of online mirror descent with a Bregman divergence and then specialize to the entropic regularization case. Let $\psi(w)$ be a convex fuction and $\psi^*(\theta)$ its Fenchel conjugate. Define the Bregman divergence $B_{\psi}(w, w')=  \psi(w) - \psi(w') - \nabla \psi(w')^T(w-w')$. In the following we use the subscript $\cdot_t$ instead of $(\cdot)(t)$ for clarity. The standard online mirror descent learner sets

\begin{equation*}
w_t = \argmin_w \left ( \gamma_t^T w + \frac{1}{\eta}B_{\psi}(w, w_t) \right ).
\end{equation*}
Using optimality of $w_{t+1}$ in the preceding equation, we have
\begin{align*}
\gamma_t^T (w_t-w^*) &= \gamma_t^T(w_{t+1}-w^*) + \gamma_t^T(w_t-w_{t+1})\\
&\le \frac{1}{\eta}(\nabla\psi(w_{t+1})-\nabla\psi(w_t))^T(w^*-w_{t+1})\\
&+ \gamma_t^T(w_t-w_{t+1})\\
&= \frac{1}{\eta} \left ( B_{\psi}(w^*, w_t) - B_{\psi}(w^*, w_{t+1}) - B_{\psi}(w_{t+1}, w_t)  \right )\\
&+ \gamma_t^T(w_t-w_{t+1}).
\end{align*}
Summing this preceding display over iterations $t$ yields
\begin{align*}
\sum_{t=1}^T \gamma_t^T(w_t-w^*) & \le \frac{1}{\eta}B_{\psi}(w^*, w_1) \\
&+ \sum_{t=1}^T \left ( -\frac{1}{\eta}B_{\psi}(w_{t+1}, w_t)+ \gamma_t^T(w_t-w_{t+1})\right)
\end{align*}
Now let $\psi(w) = \sum_i w_i \log w_i$. Then, with $w_1 = \mathbf{1}/d$, $B{\psi}(w^*, w_1) \le \log(d)$. Now we bound the second term with the following lemma.

\begin{lemma}\label{thm:bregman-bound}
Let $\psi(x)=\sum_{j} x_{j} \log x_{j}$ and $x, y \in \Delta$ be defined by:
$y_{i}=\frac{x_{i} \exp \left(-\eta g_{i}\right)}{\sum_{j} x_{j} \exp \left(-\eta g_{j}\right)}$ where $g \in \R^d_+$ is non-negative. Then
\begin{equation*}
	-\frac{1}{\eta} B_{\psi}(y, x)+g^T(x-y) \leq \frac{\eta}{2} \sum_{i=1}^{d} g_{i}^{2} x_{i}.
\end{equation*}
\end{lemma}
See Appendix \ref{sec:bregman-proof} for the proof.
Setting $y=w_{t+1}$, $x=w_{t}$, and $g=\gamma_t$ in Lemma \ref{thm:bregman-bound} yields
\begin{equation*}
\sum_{t=1}^T \gamma_t^T(w_t-w^*) \le \frac{\log(d)}{\eta} + \frac{\eta}{2} \sum_{t=1}^T\sum_{j=1}^d w_j(t)\gamma_j(t)^2.
\end{equation*}
Taking expectations on both sides yields the result. 

\subsection{Proof of Lemma \ref{thm:bregman-bound}}\label{sec:bregman-proof}
Note that $B_{\psi}(y,x) = \sum_i y_i  \log\frac{y_i}{x_i}$. Substituting the values for $x$ and $y$ into this expression, we have
\begin{equation*}
\sum_{i} y_{i} \log \frac{y_{i}}{x_{i}}=-\eta g^Ty-\sum_{i} y_{i} \log \left(\sum_{j} x_{j} e^{-\eta g_{j}}\right)
\end{equation*}

Now we use a Taylor expansion of the function $g \mapsto \log \left(\sum_{j} x_{j} e^{-\eta g_{j}}\right)$ around the point $0$. If we define the vector $p_{i}(g)=x_{i} e^{-\eta g_{i}} /\left(\sum_{j} x_{j} e^{-\eta g_{j}}\right)
$, then 
\begin{align*}
\log \left(\sum_{j} x_{j} e^{-\eta g_{j}}\right)=&\log (\mathbf{1}^Tx)-\eta p(0)^Tg +\\
&\frac{\eta^{2}}{2} g^{\top}\left(\operatorname{diag}(p(\widetilde{g}))-p(\widetilde{g}) p(\widetilde{g})^{\top}\right) g
\end{align*}
where $\widetilde{g}= \lambda g$ for some $\lambda \in \left[0,1\right]$. Noting that $p(0) = x$ and $\mathbf{1}^Tx=\mathbf{1}^Ty = 1$, we obtain
\begin{equation*}
	B_{\psi}(y, x)=\eta g^T(x-y)-\frac{\eta^{2}}{2} g^{\top}\left(\operatorname{diag}(p(\widetilde{g}))-p(\widetilde{g}) p(\widetilde{g})^{\top}\right) g,
\end{equation*}
whereby
\begin{align}\label{eq:temp-bound}
-\frac{1}{\eta} B_{\psi}(y, x)+g^T(x-y) \leq \frac{\eta}{2} \sum_{i=1}^{d} g_{i}^{2} p_{i}(\widetilde{g}).
\end{align}
Lastly, we claim that the function
\begin{align*}
s(\lambda)=\sum_{i=1}^{d} g_{i}^{2} \frac{x_{i} e^{-\lambda g_{i}}}{\sum_{j} x_{j} e^{-\lambda g_{j}}}
\end{align*}
is non-increasing on $\lambda \in [0,1]$. Indeed, we have
\begin{align*}
s^{\prime}(\lambda)&=\frac{\left(\sum_{i} g_{i} x_{i} e^{-\lambda g_{i}}\right)\left(\sum_{i} g_{i}^{2} x_{i} e^{-\lambda g_{i}}\right)}{\left(\sum_{i} x_{i} e^{-\lambda g_{i}}\right)^{2}}-\frac{\sum_{i} g_{i}^{3} x_{i} e^{-\lambda g_{i}}}{\sum_{i} x_{i} e^{-\lambda g_{i}}}\\
&=\frac{\sum_{i j} g_{i} g_{j}^{2} x_{i} x_{j} e^{-\lambda g_{i}-\lambda g_{j}}-\sum_{i j} g_{i}^{3} x_{i} x_{j} e^{-\lambda g_{i}-\lambda g_{j}}}{\left(\sum_{i} x_{i} e^{-\lambda g_{i}}\right)^{2}}
\end{align*}

Using the Fenchel-Young inequality, we have $a b \leq \frac{1}{3}|a|^{3}+\frac{2}{3}|b|^{3 / 2}$ for any $a,b$ so $g_{i} g_{j}^{2} \leq \frac{1}{3} g_{i}^{3}+\frac{2}{3} g_{j}^{3}$. This implies that the numerator in our expression for $s^{'}(\lambda)$ is non-positive. Thus,  $s(\lambda) \leq s(0)=\sum_{i=1}^{d} g_{i}^{2} x_{i}$ which gives the result when combined with inequality \eqref{eq:temp-bound}.
\section{Hardware}
\label{sec:hardware}

The major components of the vehicle used in experiments are shown in Figure~\ref{fig:car}.
The chassis of the 1/10-scale vehicles used in experiments are based on a Traxxas Rally 1/10-scale radio-controlled car with an Ackermann steering mechanism.
An electronic speed controller based on an open source design \cite{vedder} controls the RPM of a brushless DC motor and actuates a steering servo.
A power distribution board manages the power delivery from a lithium polymer (LiPo) battery to the onboard compute unit and sensors.
The onboard compute unit is a Nvidia Jetson Xavier, a system-on-a-chip that contains 8 ARM 64 bit CPU cores and a 512 core GPU. 
The onboard sensor for localization is a planar LIDAR that operates at 40Hz with a maximum range of 30 meters. The electronic speed controller also provides odometry via the back EMF of the motor.

\begin{figure}
	\centering
	\ifarxiv
	\includegraphics[width=0.5\columnwidth]{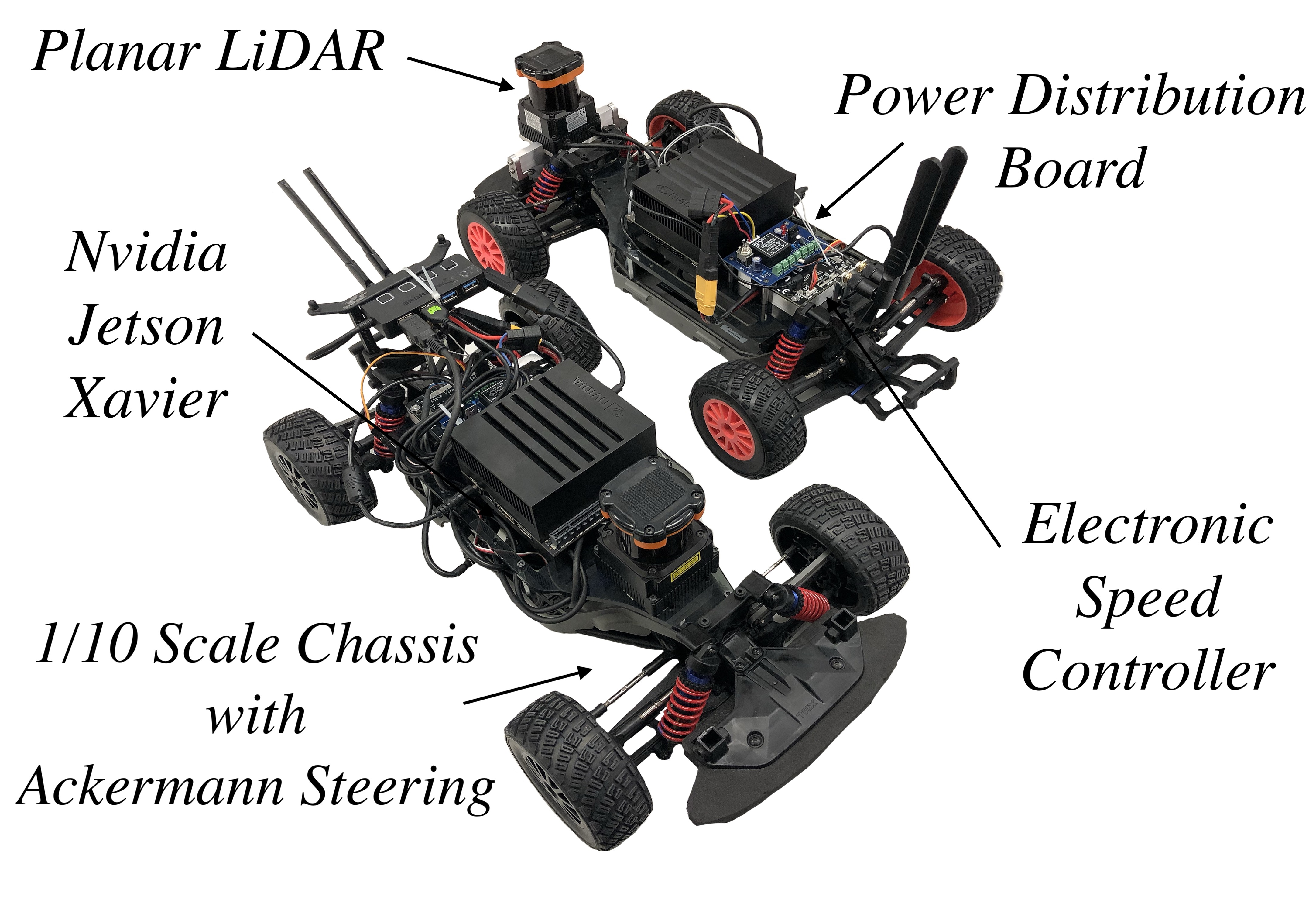}
	\else
	\includegraphics[width=0.8\columnwidth]{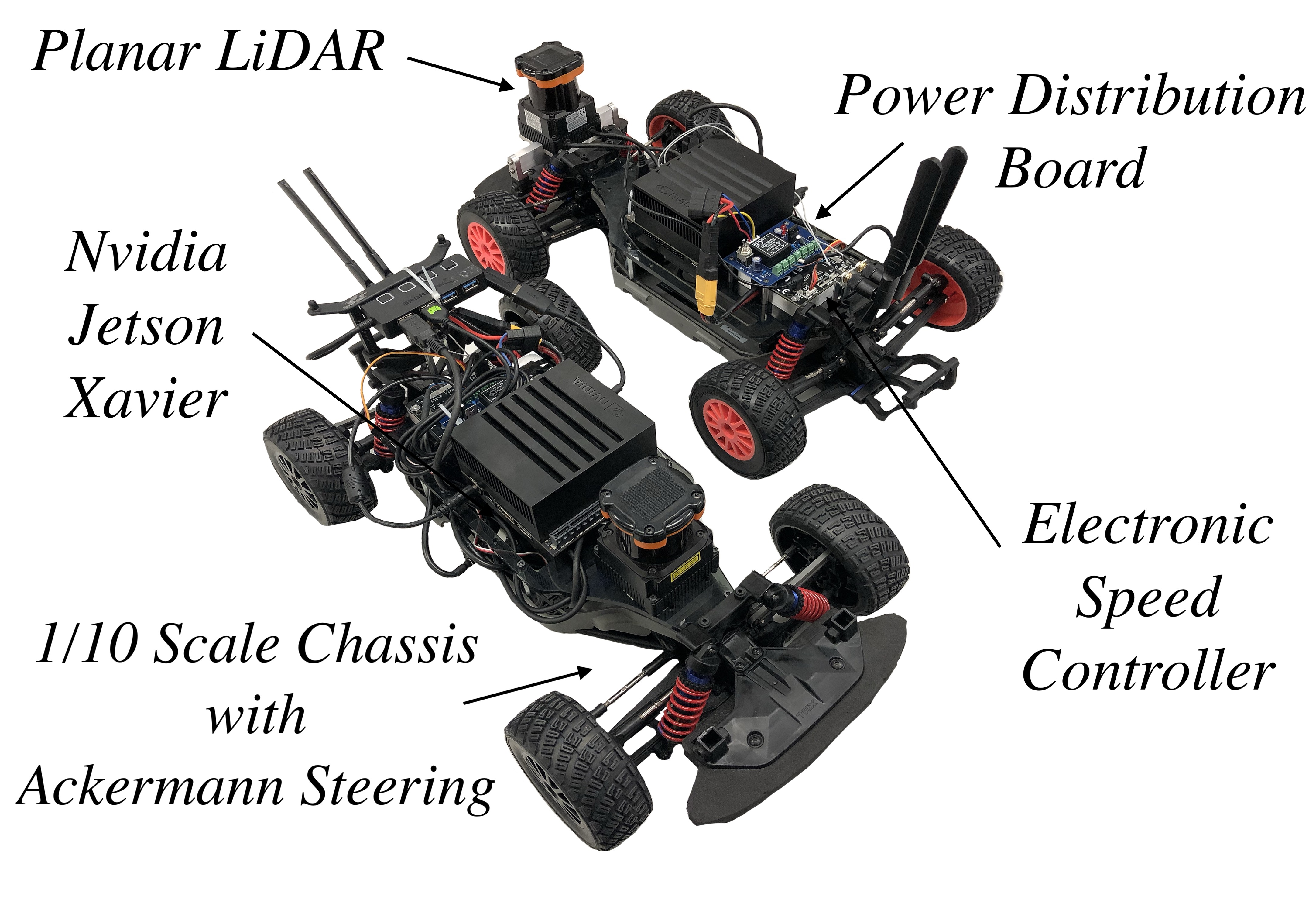}
	\fi
	\caption{Components of the 1/10 Scale Vehicle}
	\label{fig:car}
\end{figure} %
\section{Vehicle Software Stack}
\label{sec:software}
This section gives a detailed overview of the software used onboard the vehicles. Figure~\ref{fig:system} gives a graphical overview. 

\subsection{Mapping}
We create occupancy grid maps of tracks using Google Cartographer \cite{hess2016real}. The map's primary use is as an efficient prior for vehicle localization algorithms. In addition, maps serve as a representation of the static portion of the simulation environment describing where the vehicle may drive and differentiating which (if any) portions of the LIDAR scan have line-of-sight to other agents. A feature of our system useful to other researchers is that any environment which can be mapped may be trivially added to the simulator described in~Appendix \ref{sec:simulator}.
\subsection{Localization}
Due to the speeds at which the vehicles travel, localization must provide pose estimates at a rate of at least 20 Hz. Thus, to localize the vehicle we use a particle filter~\cite{walsh17} that implements a ray-marching scheme on the GPU in order to efficiently simulate sensor observations in parallel. We add a small modification which captures the covariance of the pose estimate. We do not use external localization systems (\eg~motion capture cameras) in any experiment.  
\begin{figure*}
	\centering
	\includegraphics[width=0.7\textwidth]{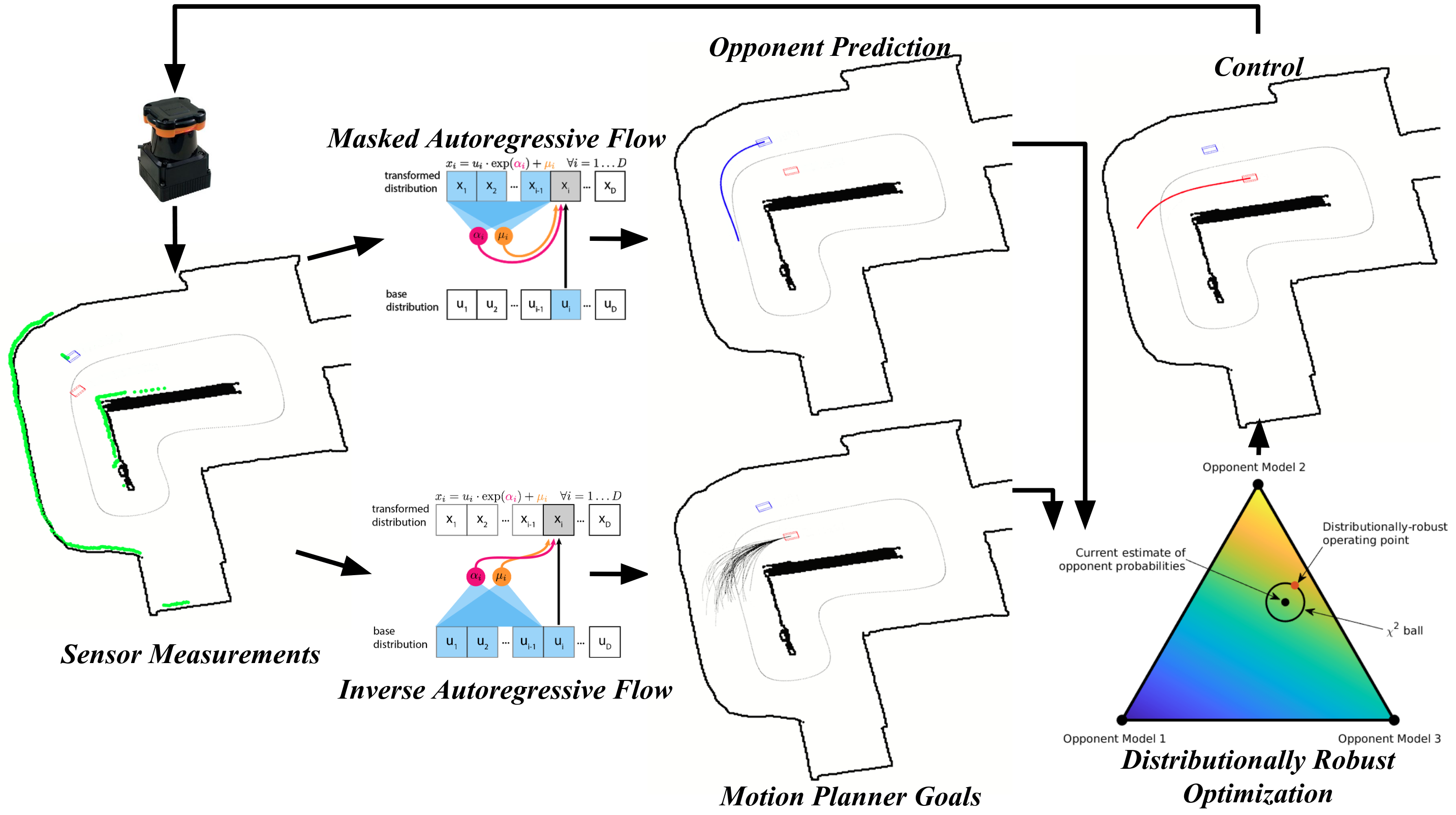}
	\caption{FormulaZero implementation on vehicle. Online each agent measures the world using onboard sensors such as a planar LIDAR. Given the sensor measurement the vehicle performs opponent prediction via the use of a masked autoregressive flow and simultaneously selects motion planner goals using an inverse autoregressive flow. Given the set of goals each is evaluated within our DRO framework, the best goal is chosen, and a new control command is applied to the vehicle. Then, the process occurs again.}
	\label{fig:system}
\end{figure*}

\subsection{Planning}

The vehicle software uses a hierarchical planner \cite{gat1998three} similar to that of \citet{ferguson2008motion}. At the top level the planner receives a map and waypoints representing the centerline of the track; the goal is to traverse the track from start to finish. Unlike route planning in road networks, there are no routing decisions to be made. In more complex instances of our proposed environment, this module could be necessary for making strategic decisions such as pit stops. The second key difference is the mid-level planner. Whereas \citet{ferguson2008motion} uses a deterministic lattice of points, our vehicle draws samples from a neural autoregressive flow. Each sample contains a goal pose and speed profile. Given this specification, the local planner calculates a trajectory parameterized as a cubic spline, evaluates static and dynamic costs of the proposed plan in belief space, and selects the lowest cost option.   

\subsubsection{Sampling behavior proposals}
There are two advantages to using a neural autoregressive flow in our planning framework. First, each agent in the population weights the individual components of its cost function differently; the flow enables the goal generation mechanism to learn a distribution which places more probability mass on the agent's preferences. Second, as planning takes place in the context of the other agent's actions, the ego-agent's beliefs can be updated by inverting the flow and estimating the likelihood of the other agent's actions under a given configuration of the cost function.

The goal-generation process utilizes an inverse autoregressive flow (IAF)~~\cite{kingma2016improved}.
The IAF samples are drawn from a density conditioned on a 101-dimensional observation vector composed of a subsampled LIDAR scan and current speed. Each sample is a 6 dimensional vector:
$\Delta t$, the perpendicular offset of the goal pose from the track's centerline; $\Delta s$, the arc-length along the track's centerline relative to the vehicle's current pose; $\Delta \theta$, the difference between the goal pose's heading angle and the current heading angle; three velocity offsets from the vehicle's current velocity at three equidistant knot points along the trajectory. 

The second benefit of using a generative model for sampling behavior proposals is the ability to update an agent's beliefs about the opponent's policy type. As noted in Section~\ref{sec:experiments}, masked~\cite{papamakarios2017masked} and inverse autoregressive flows (MAF and IAF respectively) have complementary strengths. While sampling from a MAF is slow, density estimation using this architecture is fast. Thus, we use a MAF network trained to mimic the samples produced by the IAF for this task. The architectures of each network are the same, and we describe this architecture below. 

The IAF and MAF networks used in this paper have 5 MADE layers~\cite{papamakarios2017masked} each containing: a masked linear mapping ($\R^6\to\R^{100}$), RELU layer, masked linear mapping ($\R^{100} \to\R^{100}$), RELU layer, and a final masked linear layer ($\R^{100}\to\R^{12}$). Note that output of a MADE layer includes both the transformed sample and the logarithm of the absolute value of the determinant of the Jacobian of the transformation. For sampling, the latter is discarded, and the transformed sample is passed to the next layer.  
In addition, the masking pattern is sequential and held constant during both training and inference. This choice was made to aid in debugging of experiments and to simplify communication during distributed training. 

Each population member has a dedicated IAF model, which is trained iteratively according to the \textsc{AAdaPT} algorithm described in Section~\ref{sec:pbt} using the hyperparameters given in Section~\ref{sec:experiments}. We initialize each IAF with a set of weights which approximate an identity transformation for random pairs of samples from a normal distribution and simulated observations. In addition each population member also has a MAF model, which is trained using the same hyperparameters as the IAF but only after \textsc{AAdaPT} has finished.
The code submitted in the supplementary materials extends an existing library\footnote{\url{https://github.com/kamenbliznashki/normalizing_flows}} created by other authors; we add support for the IAF architecture as well as generalize the network architecture to 3-dimensional tensors. The latter extension enables sampling from multiple agents' IAF models simultaneously and efficiently.

\subsubsection{Model Predictive Control}

The goal of the trajectory generator is to compute kinematically and dynamically feasible trajectories that take the vehicle from its current pose to a set of sampled poses from the IAF. The trajectory generator combines approaches from \cite{howard2009adaptive, nagy2001trajectory, kelly2003reactive, mcnaughton2011parallel}. Each trajectory is represented by a cubic spiral with five parameters $p=[s, a, b, c, d]$ where $s$ is the arc length of the spiral, and $(a, b, c, d)$ encode the curvature at equispaced knot points along the trajectory. Powell's method or gradient descent can be used to find the spline parameters that (locally) minimize the sum of the Euclidean distance between the desired endpoint pose and the forward simulated pose. Offline, a lookup table of solutions for a dense grid of goal poses is precomputed, enabling fast trajectory generation online. Each trajectory is associated with an index which selects the $\Delta x$, $\Delta y$, and the $\Delta \theta$ of the goal pose relative to the current pose (where positive $x$ is ahead of the vehicle and postiive $y$ is to the left), and $\kappa_0$, the initial curvature of the trajectory. The resolution and the range of the table is listed in Table \ref{table:lutreso}. Figure~\ref{fig:samples} shows a selection of trajectories. The point on the left of the figure is the starting pose of the vehicle, and the collection of goal poses is shown as the points on the right of the figure.

\begin{table}
	\caption{The resolution and ranges of the Trajectory Generator Look-up Table}
	\label{table:lutreso}
	\begin{center}
		\begin{small}
			\begin{tabular}{lccc}
				\toprule
				Index     & Resolution     & Min     & Max \\
				\midrule
				$\Delta x$ & 0.1 m & -1.0 m &  10.0 m \\
				$\Delta y$ & 0.1 m & -8.0 m &  8.0 m \\
				$\Delta \theta$   & $\pi/32$ rad & $-\pi/2$ rad & $\pi/2$ rad \\
				$\kappa_0$  & 0.2 rad/m & -1.0 rad/m & 1.0 rad/m \\
				\bottomrule
			\end{tabular}
		\end{small}
	\end{center}
	\vskip -0.2in
\end{table}

\begin{figure}
	\centering
	\ifarxiv
	\includegraphics[width=0.6\columnwidth]{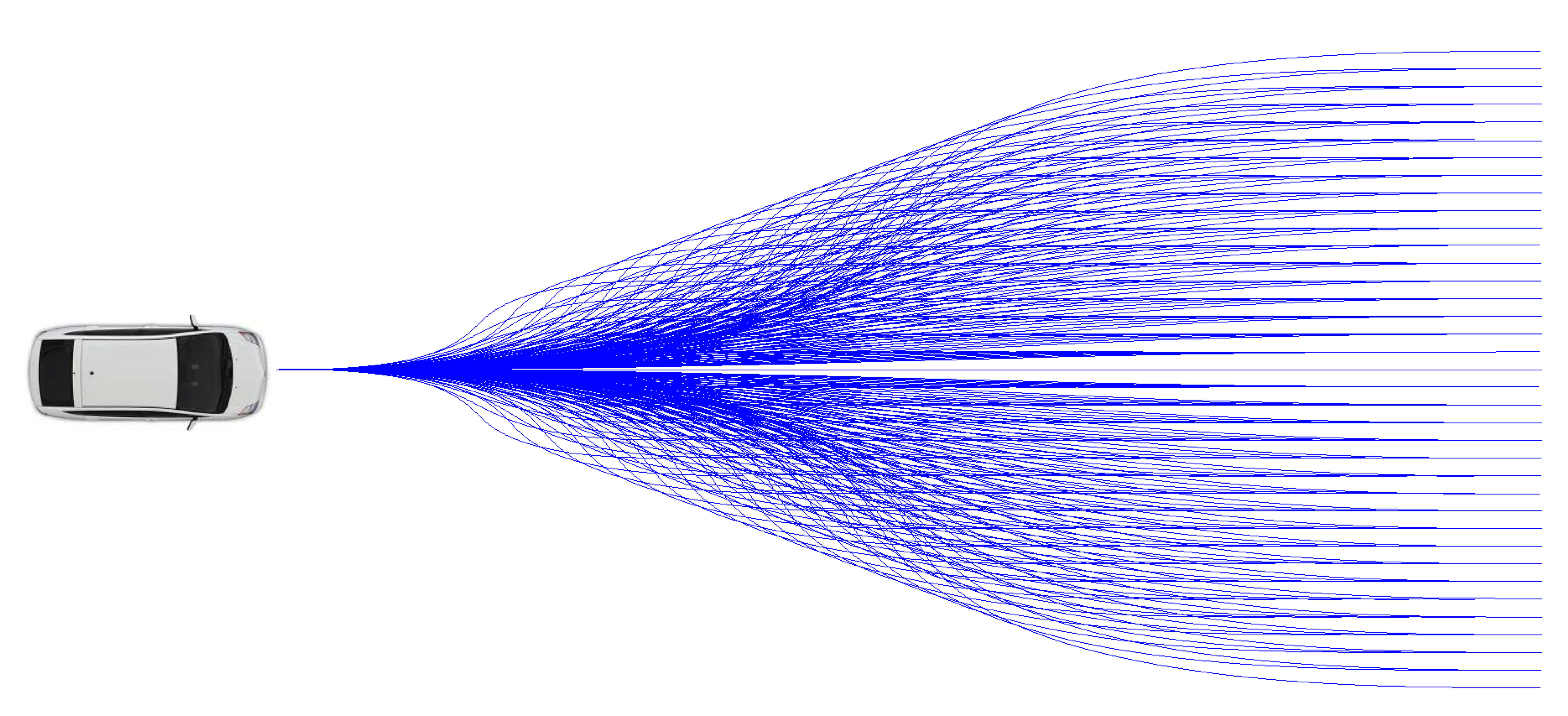}
	\else
	\includegraphics[width=0.9\columnwidth]{figs/traj_gen.png}
	\fi
	\caption{Sample trajectories from the look-up table}
	\label{fig:samples}
\end{figure}

\subsubsection{Trajectory Cost Functions}
Each of the generated trajectories is evaluated with the weighted sum of the following cost functions. Note, in order to ensure safety, goals which would result in collision result in infinite cost and are automatically rejected prior to computing the robust cost, which operates only on finite-cost proposals.

\begin{enumerate}
	\item \textbf{Trajectory length:} $c_{al}=s$, where $1/s$ is the arc length of each trajectory. Short and myopic trajectories are penalized.
	
	\item \textbf{Maximum absolute curvature:} $c_{mc}=\max_i\{|\kappa_i|\}$, where $\kappa_{i}$ are the curvatures at each point on a trajectory. Large curvatures are penalized to preserve smoothness of trajectories.
	
	\item \textbf{Mean absolute curvature:} $c_{ac}=\frac{1}{N}\sum_{i=0}^{N}|\kappa_{i}|$, the notation is the same as $c_{mc}$ and the effect of this feature is similar, but less myopic. 
	
	\item \textbf{Hysteresis loss:} Measured between the previous chosen trajectory and each of the sampled trajectories, $c_{hys}=||\theta_{prev}^{[n_1,n_2]} - \theta^{[0,n_2-n_1]}||_2^2$, where $\theta_{prev}$ is the array of heading angles of each pose on the previous selected trajectory by the vehicle, $\theta$ is the array of heading angles of each pose on the trajectory being evaluated, and the ranges $[n_1,n_2]$ and $[0, n_2-n_1]$ define contiguous portions of trajectories that are compared. Trajectories dissimilar to the previously selected trajectory are penalized.

	\item \textbf{Lap progress:} Measured along the track from the start to the end point of each trajectory in the normal and tangential coordinate system, $c_p=\frac{1}{s_{end}-s_{start}}$, where $s_{end}$ is the corresponding position in the tangential coordinate along the track of the end point of a trajectory, and $s_{start}$ is that of the start point of a trajectory. Shorter progress in distance is penalized.
	
	\item \textbf{Maximum acceleration:} $c_{ma}=\max_i|\frac{\Delta v_i}{\Delta t_i}|$ where $\Delta v$ is the array of difference in velocity between adjacent points on a trajectory, and $\Delta t$ is the array of corresponding time intervals between adjacent points. High maximum acceleration is penalized.
	
	\item \textbf{Maximum absolute curvature change:} Measured between adjacent points along each trajectory, $c_{dk}=\max_i|\frac{\Delta \kappa_i}{\Delta t_i}|$. High curvature changes are penalized.
	
	\item \textbf{Maximum lateral acceleration:}  $c_{la}=\max_i\{|\kappa|_iv_i^2\}$, where $\kappa$ and $v$ are the arrays of curvature and velocity of all points on a trajectory.
	High maximum lateral accelerations are penalized.
	
	\item \textbf{Minimum speed:} $c_{ms}=\frac{1}{(\min_i\{v_i\})_+}$. Low minimum speeds are penalized.
	
	\item \textbf{Minimum range:} $c_{mr}=\min_i\{r_i\}$, where $r$ is the array of range measurements (distance to static obstacles) generated by the simulator. Smaller minimum range is penalized, and trajectories with minimum ranges lower than a threshold are given infinite cost and therefore discarded.

	\item \textbf{Cumulative inter-vehicle distance short:} 
	\begin{equation*}
		c_{dyshort} = 
		\begin{cases}
			\infty, \; \text{if} \; d(\mathtt{ego}_i, \mathtt{opp}_i)\leq\mathtt{thresh}\\
			\sum_{i=0}^{N_{short}} d(\mathtt{ego}_i, \mathtt{opp}_i), \; \text{otherwise}
		\end{cases}
	\end{equation*}
	Where the function $d()$ returns the instantaneous minimum distance between the two agents at point $i$, $N_{short}$ is a point that defines the shorter time horizon for a trajectory of $N$ points. Trajectories with infinite cost on the shorter time horizon are considered infeasible and discarded.

	\item \textbf{Discounted cumulative inter-vehicle distance long:} $c_{dylong}=\sum_{i=N_{short}}^{N_{long}}0.9^{i-N_{short}}\frac{1}{d(\mathtt{ego}_i, \mathtt{opp}_i)}$, where $N_{long}$ is a point that defines the longer time horizon for a trajectory of $N$ points. Note that $N_{short}<N_{long}<N$ . Lower minimum distances between agents on the longer time horizon are penalized.
	
	\item \textbf{Relative progress:} Measured along the track between the sampled trajectories' endpoints and the opponent's selected trajectory's endpoint, $c_{dp}=(s_{opp\_end}-s_{end})_+$, where $s_{opp\_end}$ is the position along the track in tangential coordinates of the endpoint of the opponent's chosen trajectory. Lagging behind the opponent is penalized.

\end{enumerate}

\subsubsection{Path tracker}
Once a trajectory has been selected it is given to the path-tracking module. The goal of the path tracker is to compute a steering input which drives the vehicle to follow the desired trajectory. Our implementation uses a simple and industry-standard geometrical tracking method called pure pursuit \cite{coulter1992implementation, snider2009automatic}. Due to the decoupling of the trajectory generation and tracking modules it is possible for the tracker to run at a much higher frequency than the trajectory generator; this is essential for good performance. 

\subsection{Communication and system architecture}

The ZeroMQ \cite{hintjens2013zeromq} messaging library is used to create interfaces between the FormulaZero software stack and the underlying ROS nodes that control and actuate the vehicle test bed. Unlike in the simulator, some aspects of the FormulaZero planning function operate non-deterministicaly and asynchronously. In particular we use a sink node to collect observations from ROS topics related to the various sensors on the vehicle in order to approximate the step-function present in the Gym API. When a planning cycle is complete, the trajectory is published back to ROS and tracked asynchronously using pure-pursuit  as new pose estimates become available.  
Because perception is not the primary focus of this project we simplify the problem of detecting and tracking the other vehicle. In particular, each vehicle estimates its current pose in the map obtained by its onboard particle filter, and this information is communicated to the other vehicle via ZeroMQ over a local wireless network. Since tracking and detection has been well studied in robotics, solutions which rely less on communication could be explored by other future work which builds upon this paper.

\section{Simulation Stack}
\label{sec:simulator}
The simulation stack includes a lightweight 2D physics engine with a dynamical vehicle model. Then on top of the physics engine, a multi-agent simulator with an OpenAI Gym \cite{gym} API is used to perform rollouts of the experiments.
\subsection{Vehicle Dynamics}
The single-track model in \citet{althoff2017commonroad} is chosen because it considers tire slip influences on the slip angle, which enables accurate simulation at physical limits of the vehicle test bed. It is also easily enables changes to the driving surface friction coefficient in simulation which allows the simulator to model a variety of road surfaces.

\subsection{System Identification}
Parameter identification was performed to derive the following vehicle parameters: mass, center of mass, moment of inertia, surface friction coefficient, tire cornering stiffness, and maximum acceleration/deceleration rates following the methods described in \citet{tunercar2019tech}.

\subsection{Distributed Architecture}
Due to the nature of the \textsc{AAdaPT} algorithm, the rollouts in a single vertical step do not need to be in sequence. The ZeroMQ messaging library is used to create a MapReduce \cite{dean2008mapreduce} pattern between the task distributor, result collector, and the workers. Each worker receives the description of the configuration to be simualted, \eg~$(x, \theta)$. Then the workers asynchronously perform simulations and send results to the collector.

\subsection{Addressing the simulation/reality gap}
As noted in Section~\ref{sec:experiments} there are several differences between the observations in simulated rollouts and reality. First, pose estimation errors are not present in the simulator. A simple fix would be to add Gaussian white noise to the pose observations returned by the simulator. We avoided this and other domain randomization techniques in order to preserve the determinism of the simulator, but we will investigate its effect in further experiments. Second, the LIDAR simulation does not account for material properties of the environment. In particular, surfaces such as glass do not produce returns, causing subsets of the LIDAR beams to be dropped. We hypothesize that simple data augmentation schemes which select a random set of indices to drop from simulated LIDAR observations would improve the robustness to such artifacts when the system is deployed on the real car; we are currently investigating this hypothesis.  

\section{Experiments}\label{sec:experiments_appendix}
Additional videos of simulation runs are available.\footnote{\url{https://youtu.be/8q0lZssbEI4}}

\subsection{Instantaneous time-to-collision (iTTC)}\label{sec:ittc}
Let $T_i(t)$ be the instantaneous time-to-collision between the ego vehicle and the $i$-th environment vehicle at time step $t$. The value $T_i(t)$ can be defined in multiple ways (see~\eg~\citet{sontges2018worst}). \citet{norden2019efficient} define it as the amount of time that would elapse before the two vehicles' bounding boxes intersect assuming that they travel at constant fixed velocities from the snapshot at time $t$. Time-to-collision captures directly whether or not the ego-vehicle was involved in a crash. If it is positive no crash occurred, and if it is 0 or negative there was a collision.

\subsection{Out-of-distribution agent strategies}\label{sec:oodagent}
In the following sections, we describe the human-created algorithms used in our out-of-distribution analysis.

\subsubsection{\textsc{OOD1}: RRT* with MPC-based Opponent Prediction}
This approach exploits the fact that the two-car racing scenario is similar to driving alone on the track with the only exception being during overtaking the opponent. This approach uses a costmap-based RRT* \cite{karaman2011sampling} planning algorithm.
The agent first uses the opponent's current pose and velocity in the world, and uses Model-Predictive Control to calculate an open loop trajectory of N optimal inputs resulting in N+1 states based on a given cost function and constraints. Specifically, the optimization problem is constrained by a linearized version of the single track model described in \citet{althoff2017commonroad}, and by the boundary values of the inputs and states of the vehicle. The cost function that the optimization tries to minimize consists of the trajectory length and input power requirement.
The costmap used by RRT* also incorporates this predicted trajectory of the opponent vehicle by inflating the two-dimensional spline representing the prediction, and weighting the portion of the spline closer to the ego vehicle higher.
RRT* samples the two dimensional space that the vehicle lies in. The path generated by RRT* is then tracked with the Pure Pursuit controller \cite{coulter1992implementation}.

\subsubsection{\textsc{OOD2}: RL-based Lane Switching}
\begin{figure}[h!]
	\centering
	\includegraphics[scale=0.6]{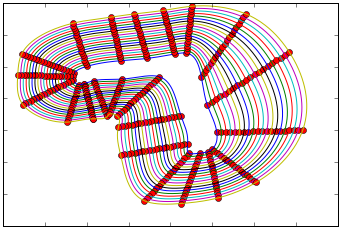}
	\caption{Lanes that cover the track}
	\label{fig:rlpaths}
\end{figure}

The second algorithm is based on a lane-switching planning strategy that uses an RL algorithm to make lane switching decisions, and filters out unsafe decisions using a collision indicator.
First, as shown in \ref{fig:rlpaths}, different lanes going through numerous checkpoints on the track are created to cover the entirety of the race track.
Then a network is trained to make lane switching decisions. The state of the RL problem consists of the sub-sampled LIDAR scans of the ego vehicle; the pose ($x, y, \theta$) of the opponent car with respect to the ego vehicle; velocity ($v_x, v_y$) of the opponent vehicle with respect of the ego vehicle; projected distance from the ego vehicle's current position to all pre-defined paths. The reward of a rollout is zero in the beginning. At each timestep, the timestep itself is subtracted from the total reward. A rollout receives -100 as the reward when the ego agent collide with the environment or the other agent. And finally, if both agents finish 2 laps, the difference between lap times (positive if the ego agent wins) of the two agents are added to the reward. Clipped Double Q-Learning \cite{fujimoto2018addressing} is used to estimate the Q function and make the lane switching decisions.
iTTC defined in Appendix \ref{sec:ittc} is used as an indicator for future collisions. If any decisions made by the RL network would result in a collision indicated by the iTTC value, the safety function kicks in and makes the lane switching decision based on the collision indicator.
Finally, ego vehicle actuation is provided by the same Pure Pursuit controller \cite{coulter1992implementation} tracking the selected lane. We used an existing implementation\footnote{\url{https://github.com/pnorouzi/rl-path-racing}} of this algorithm. 
\end{document}